\documentclass{article}

\usepackage{microtype}
\usepackage{graphicx}
\usepackage{subfigure}
\usepackage{booktabs} 
\usepackage{mathrsfs}

\usepackage{nicefrac}
\usepackage{xcolor}
\usepackage{amsmath,amssymb,amsthm}
\usepackage{thm-restate}

\theoremstyle{definition}

\usepackage{enumitem}
\setlist[enumerate]{leftmargin=0.5cm,topsep=0pt,itemsep=-2pt}
\setlist[itemize]{leftmargin=0.5cm,topsep=0pt,itemsep=-2pt}

\usepackage{mathrsfs}
\newcommand{\probspace}{\mathscr{P}}

\usepackage{mathtools}
\usepackage{hyperref}
\hypersetup{
    colorlinks=true,
    linkcolor=blue,
    citecolor=cyan,
}

\usepackage[capitalize]{cleveref}
\Crefname{assumption}{Assumption}{Assumptions}
\Crefname{equation}{Equation}{Equations}
\creflabelformat{equation}{(#2#1#3)}

\usepackage[accepted]{icml2021}

\icmltitlerunning{Towards a Better Understanding of Representation Dynamics under TD-learning}

\begin{document}

\twocolumn[
\icmltitle{Towards a Better Understanding of Representation Dynamics under TD-learning}



\icmlsetsymbol{equal}{*}

\begin{icmlauthorlist}
\icmlauthor{Yunhao Tang}{dm}
\icmlauthor{R\'emi Munos}{dm}
\end{icmlauthorlist}

\icmlaffiliation{dm}{Google DeepMind}

\icmlcorrespondingauthor{Yunhao Tang}{robintyh@deepmind.com}

\icmlkeywords{Machine Learning, ICML}

\vskip 0.3in
]



\printAffiliationsAndNotice{\icmlEqualContribution} 

\begin{abstract}
TD-learning is a foundation reinforcement learning (RL) algorithm for value prediction. Critical to the accuracy of value predictions is the quality of state representations. In this work, we consider the question: how does end-to-end TD-learning impact the representation over time? Complementary to prior work, we provide a set of analysis that sheds further light on the representation dynamics under TD-learning. We first show that when the environments are reversible, end-to-end TD-learning strictly decreases the value approximation error over time. Under further assumptions on the environments, we can connect the representation dynamics with spectral decomposition over the transition matrix. This latter finding establishes fitting multiple value functions from randomly generated rewards as a useful auxiliary task for representation learning, as we empirically validate on both tabular and Atari game suites.
\end{abstract}

\section{Introduction}

Temporal difference (TD) learning is a foundational algorithm for predicting value functions in reinforcement learning (RL) \citep{sutton1988learning}. In practice, computations of the value predictions depend on how the state is represented, a quantity formally known as state \emph{representation}. In classic settings such as linear TD, the representations are often human designed and fixed throughout learning. In this case, the quality of the best possible value predictions depends critically on the quality of the fixed representations. Good representations should share useful information across states that entail accurate value predictions (see, e.g., \citep{tsitsiklis1996analysis,munos2003error,behzadian2019fast} for some example characterizations).

Any fixed representation is potentially sub-optimal, as it does not adapt to the underlying learning algorithm. To alleviate such issues, in practice, it is now common to perform gradient-based updates on the representations also with TD-learning. Recently, such an \emph{end-to-end} approach 
to representation learning has led to much empirical success and is the core of many high-performing agents such as DQN \citep{mnih2013}. A natural question ensues: can we characterize the representation learned by such end-to-end updates?

The answer to this question has been attempted by a number of prior work, including the study of the convergence of end-to-end TD-learning under the over-parameterized regimes, i.e., when the value functions are learned by very wide neural networks \citep{cai2019neural,zhang2020can,agazzi2022temporal,sirignano2022asymptotics}; the study of TD-learning dynamics under smooth homogeneous function approximation, e.g., with ReLU networks \citep{brandfonbrener2019geometric}; the study of representation dynamics under TD-learning with restrictive assumptions on the weight parameter \citep{lyle2021effect}. See \cref{sec:discussion} for an in-depth discussion about this paper's relation to prior work.

In this work, we provide a set of analysis complementary to prior work, which hopefully sheds light on how TD-learning  impacts the evolution of representation over time. We consider the natural extension of the linear TD case \citep{sutton1988learning,tsitsiklis1996analysis}, where the value function is parameterized in a bi-linear way 
\begin{align*}
   V = \Phi w.
\end{align*}
Here, $\Phi$ is the representation, which is kept fixed in linear TD. We study the evolution of the representation under the general TD-learning updates. Our contributions consist in providing characterizations of the representation dynamics, through a few angles. 
\paragraph{Improving value prediction accuracy over time.} When assuming the Markov chain is reversible, we show that the value prediction error, as measured by the difference $\Phi w - V^\pi$, strictly decreases over time (\cref{sec:value}). This solidifies the intuition that
allowing updating representations during TD-learning should improve upon classic TD-learning where representations are fixed.

\paragraph{Spectral decomposition of the transition matrix .} We show that when the transition matrix is symmetric and under certain conditions of reward functions,  TD-learning learns provably useful representations (\cref{sec:spectral}). Formally, the representation dynamics executes gradient-based spectral decomposition on the transition matrix. 

\paragraph{Random value predictions learn useful representations.} 
As a corollary of the previous results, we show that predicting multiple value functions generated via randomly sampled rewards is an auxiliary task that helps learn useful representations. We validate this theoretical insight with tabular and deep RL experiments over Atari game suites.

\section{Background}\label{sec:background}

Consider a Markov decision process (MDP) represented as the tuple $\left(\mathcal{X},\mathcal{A},p,p_R,\gamma\right)$ where $\mathcal{X}$ is a finite state space, $\mathcal{A}$ the finite action space,  $p:\mathcal{X}\times\mathcal{A}\rightarrow\probspace(\mathcal{X})$ the transition kernel, $p_R:\mathcal{X}\times\mathcal{A}\rightarrow\mathbb{R}^h$ the reward kernel, and $\gamma\in [0,1)$ the discount factor. In the traditional setting, the reward is scalar such that $h=1$, though in general it can be extended to multiple dimensions $h\geq 1$ \citep{sutton2011horde}. Let $\pi:\mathcal{X}\rightarrow\probspace(\mathcal{A})$ be a fixed policy. For convenience, let $P^\pi:\mathcal{X}\rightarrow\probspace(\mathcal{X})$ be the state transition kernel induced by the policy $\pi$ such that $P^\pi(x,y)=\sum_a \pi(a|x)p(y|x,a)$. Denote the state-dependent reward function as $R^\pi\in\mathbb{R}^{|\mathcal{X}|\times h}$ such that $R^\pi(x)=\sum_a r(x,a)\pi(a|x)\in\mathbb{R}^h$. Throughout, we will focus on the Markov reward process under $P^\pi$ and $R^\pi$ as this simplifies discussions. 

When the context is clear, we overload the notation and let $x$ be an one-hot encoding of the state too. A state representation $\phi_x\in\mathbb{R}^k$ is defined by a mapping from the state space $\mathcal{X}$ to the $k$-dimensional Euclidean space.  To characterize representations, we consider a matrix $\Phi\in\mathbb{R}^{|\mathcal{X}| \times k}$ from which the representation for state $x$ can be calculated as $\phi_x\coloneqq\Phi^T x\in\mathbb{R}^k$. 
Throughout, we assume $k\leq|\mathcal{X}|$ where $|\mathcal{X}|$ is the cardinal of $\mathcal{X}$, as this tends to be the case in practice. 
Good representations should entail sharing information between states, and facilitate downstream tasks such as policy evaluation or control. In classic TD-learning settings, representations are fixed, whereas in practice, representations are also shaped by incremental updates, as we detail below.
 
\subsection{TD-learning with linear function approximations}
Given the representation $\phi_x\in\mathbb{R}^k$ at state $x$, TD-learning with linear function approximation parameterizes a linear function with weight $w\in\mathbb{R}^{k \times h}$, such that the prediction $\phi_x^Tw$ approximates the value function $V^\pi\coloneqq(I-\gamma P^\pi)^{-1}R^\pi\in\mathbb{R}^{|\mathcal{X}|\times h}$.  In most applications, there is a single reward $h=1$; here, we consider the most general case $h\geq 1$ as this helps facilitate ensuing discussions.
For any $1\leq  i\leq h$, we can understand the $i$-th column of the weight vector $w_i$ as approximating the value function $V_i^\pi\in\mathbb{R}^{|\mathcal{X}|}$ derived from the $i$-th column of the reward matrix $R_i^\pi\in\mathbb{R}^{|\mathcal{X}|}$. 

The aim of TD-learning is to adjust the weight parameter $w$ such that the approximation is accurate. We start with the classic linear TD-learning setting where the representation $\Phi$ is fixed.

Given a state $x_t$ and its sampled next state $x_{t+1}\sim P^\pi(\cdot|x_t)$, TD-learning constructs the bootstrapped back-up target $R^\pi(x_t) + \gamma \phi_{x_{t+1}}^T w$ and seeks to minimize the squared prediction error 
\begin{align*}
    \left\lVert R^\pi(x_t) + \gamma\phi_{x_{t+1}}^Tw-\phi_{x_t}^Tw\right\rVert_2^2. 
\end{align*} 
TD-learning is defined through the \emph{semi-gradient} update, 
which can be understood as gradient descent on the modified squared prediction error with a stop gradient operation on the back-up target,
\begin{align}
    L(\Phi,w) \coloneqq \left\lVert \text{sg}\left(R^\pi(x_t) + \gamma \phi_{x_{t+1}}^T w\right)-\phi_{x_t}^Tw\right\rVert_2^2. \label{eq:sg-squared-error}
\end{align} 
Here we omit the loss function's dependency on the sampled transition for simplicity. Then, the weight is updated as
\begin{align}
    w_{t+1} =  w_t - \eta_w\partial_{w_t}L(\Phi,w_t),\label{eq:linear-td}
\end{align}
with learning rate $\eta_w\geq 0$.
Throughout, we assume the state $x_t$ are drawn from the stationary distribution $d^\pi\in\mathbb{R}^{|\mathcal{X}|}$ of the Markov matrix $P^\pi$. We let $D^\pi\in\mathbb{R}^{|\mathcal{X}|\times|\mathcal{X}|}$ denote the diagonal matrix constructed from $d^\pi$. Assuming $\Phi$ is of rank $k$ and $d^\pi>0$, it has been proved that under mild conditions on the learning rate and the representation matrix, linear TD-learning converges to the unique fixed point \citep{tsitsiklis1996analysis}
\begin{align}
    w_\Phi^\ast \coloneqq \left(\Phi^T D^\pi(I-\gamma P^\pi) \Phi\right)^{-1} \Phi^T D^\pi R^\pi.\label{eq:linear-td-solution}
\end{align}

The primary approach taken by the seminal work of \citep{tsitsiklis1996analysis} is to understand  linear TD-learning through the continuous time behavior of the expected linear TD updates, characterized by an ordinary differential equation (ODE). Now, we overload the notation and let $w_t$ be the weight parameter indexed by continuous time $t\geq 0$. The continuous time behavior of linear TD is characterized by the following ODE 
\begin{align}
   \frac{dw_t}{dt} = \eta_w\cdot \Phi^T D^\pi \left(R^\pi - \left(I-\gamma P^\pi\right)\Phi w_t\right). \label{eq:linear-td-ode} 
\end{align}
Here, $\eta_w>0$ is the learning rate.
We can verify that $w_\Phi^\ast$ is the unique fixed point to \cref{eq:linear-td-ode}.

\section{End-to-end linear TD-learning: jointly updating weight and representation}
We now provide further background on a natural extension of linear TD-learning to the case where the representations are updated as well, a case we call \emph{deep} TD-learning \citep{lyle2021effect}. In end-to-end linear TD, both representation $\Phi$ and weight $w$ are updated by semi-gradient descent on the squared prediction error in \cref{eq:sg-squared-error}. The joint update on $\Phi_t$ and $w_t$ is
\begin{align}
\begin{split}\label{eq:deep-linear-td}
    w_{t+1} &=  w_t - \eta_w\partial_{w_t}L(\Phi_t,w_t), \\
    \Phi_{t+1} &= \Phi_t - \eta_\Phi \partial_{\Phi_t} L(\Phi_t,w_t) 
\end{split}
\end{align}
with learning rates $\eta_w,\eta_\Phi\geq 0$. If we interpret the combination of representation and weight parameter $\theta=(\Phi,w)$ as a whole, the end-to-end linear TD update in \cref{eq:deep-linear-td}, combined with the original linear TD update in \cref{eq:linear-td} can be understood as semi-gradient descent on the loss function $L(\Phi,w)$ with respect to $\theta$. Though end-to-end linear TD-learning brings us closer to practical implementations (e.g., DQN \citep{mnih2013}), there is no general guarantee on the behavior of the joint system $(\Phi_t,w_t)$.

We aim to understand the behavior of end-to-end linear TD by characterizing the behavior of its continous time system.
The continuous time ODE of the end-to-end linear TD updates in \cref{eq:deep-linear-td} can be formally stated as follows.
\begin{restatable}{lemma}{lemmaode}\label{lemmaode} The ODE to the end-to-end linear TD update in \cref{eq:deep-linear-td} is
\begin{align}
\begin{split} \label{eq:deep-linear-td-ode}
        \frac{dw_t}{dt} &=  \eta_w\cdot\Phi_t^T D^\pi \left(R^\pi - \left(I-\gamma P^\pi\right)\Phi_t w_t\right), \\
    \frac{d\Phi_t}{dt} &= \eta_\Phi\cdot D^\pi \left(R^\pi - \left(I-\gamma P^\pi\right)\Phi_t w_t\right) w_t^T.
\end{split}
\end{align}
\end{restatable}
A critical difference between the linear TD and end-to-end linear TD is demonstrated through the ODE systems in \cref{eq:linear-td-ode} and \cref{eq:deep-linear-td-ode}: the linear TD induces a \emph{linear} ODE in the variable $w_t$, whereas the end-to-end linear TD is a coupled non-linear ODE in the joint variable $(\Phi_t,w_t)$. This hints at the difficulty in characterizing the learning dynamics of the end-to-end linear TD as alluded to earlier. 

\paragraph{Remarks on the learning rates.} The linear TD can be understood as a special case when $\eta_\Phi=0$. In practical implementations, it is more common to update both set of parameters with the same learning rate $\eta_w=\eta_\Phi>0$. An interesting case for theoretical analysis is when the representation is updated at a much slower pace than the weight vector, i.e., $\eta_\Phi\ll\eta_w$. We study such a limiting case in \cref{sec:spectral}. Below, we consider the general case where both learning rates assume positive finite values.

\section{Value function approximation error of end-to-end linear TD} \label{sec:value}

A primary quantity of interest for TD-learning is the value approximation error $\Phi w-V^\pi$. In linear TD-learning,
given the fixed $\Phi$, the approximation error is in general finite even if $w$ converges to the fixed point $w_\Phi^\ast$. In end-to-end linear TD-learning, where the representation matrix $\Phi_t$ is adapted over time based on the semi-gradient, intuitively we should expect the approximation error to decrease over time. Below, we make a formal characterization of such an intuition.

As a measure of the approximation accuracy, we define the following weighted value approximation error $E(\Phi,w)$,
\begin{align*}
    E\coloneqq \frac{1}{2}\text{Tr}\left(\left(\Phi w-V^\pi\right)^T D^\pi(I-\gamma P^\pi) \left(\Phi w-V^\pi\right)\right),
\end{align*}
where the function $\text{Tr}(M)=\sum_{i}M_{ii}$ returns the trace of a squared matrix $M$. The error is effectively the sum of weighted norms for the columns of $\Phi w -V^\pi\in\mathbb{R}^{|\mathcal{X}|\times h}$ under the key matrix 
\begin{align*}
    D^\pi(I-\gamma P^\pi)\in\mathbb{R}^{|\mathcal{X}|\times|\mathcal{X}|}.
\end{align*}
The key matrix is positive definite (PD), see, e.g., \citep{sutton2016emphatic} for a detailed proof. As a result, the weighted error enjoys a few useful properties.
\begin{restatable}{lemma}{lemmaweightederror}\label{lemmaweightederror} The weighted error $E(\Phi,w)$ is always non-negative and evaluates to zero if and only if $\Phi w=V^\pi$.
\end{restatable}


As a first important conclusion of this work, we show that when the MDP is reversible, the error function strictly decreases over time, as long as $(\Phi_t,w_t)$ has not converged to a critical point of the ODE dynamics. Intriguingly, this is because under the reversibility assumption, the end-to-end linear TD dynamics turns out to be gradient descent on the error function.

\begin{restatable}{theorem}{theoremvalue}\label{theoremvalue} Assume the Markov chain is reversible, i..e, $D^\pi P^\pi=(P^\pi)^T D^\pi$, then
the end-to-end linear TD dynamics in \cref{eq:deep-linear-td-ode} is effectively gradient descent on the error function $E(\Phi_t,w_t)$, i.e.,
\begin{align*}
    \frac{dw_t}{dt} = -\eta_w\cdot \partial_{w_t}E(\Phi_t,w_t),\ \frac{d\Phi_t}{dt} = -\eta_\Phi\cdot \partial_{\Phi_t}E(\Phi_t,w_t).
\end{align*}
As a result, the weighted value approximation error is non-increasing $
    dE(\Phi_t,w_t)/dt\leq 0$. 
If $(\Phi_t,w_t)$ is not at a critical point of the learning dynamics, i.e. when $\frac{d\Phi_t}{dt}\neq 0$ or $\frac{dw_t}{dt}\neq 0$, then $dE(\Phi_t,w_t)/dt$ is strictly negative.
\end{restatable}

\begin{proof}
We provide a proof sketch here. The reversibility assumption establishes that the key matrix $D^\pi(I-\gamma P^\pi)$ is symmetric. Under this symmetric condition, with a few lines of algebraic manipulations, we can show that indeed the updates to $w_t$ and $\Phi_t$ follow negative gradient on the weighted error $E(\Phi_t,w_t)$. As a result
\begin{align*}
    \frac{d}{dt}E(\Phi_t,w_t) = -\left(\frac{1}{\eta_\Phi}\left\lVert \frac{d\Phi_t}{dt}\right\rVert_2^2 + \frac{1}{\eta_w} \left\lVert \frac{dw_t}{dt}\right\rVert_2^2\right),
\end{align*}
where $\left\lVert\cdot\right\rVert_2$ is the $L_2 $ norm. This concludes the proof.
\end{proof}

\begin{figure}[t]
    \centering
    \includegraphics[keepaspectratio,width=.45\textwidth]{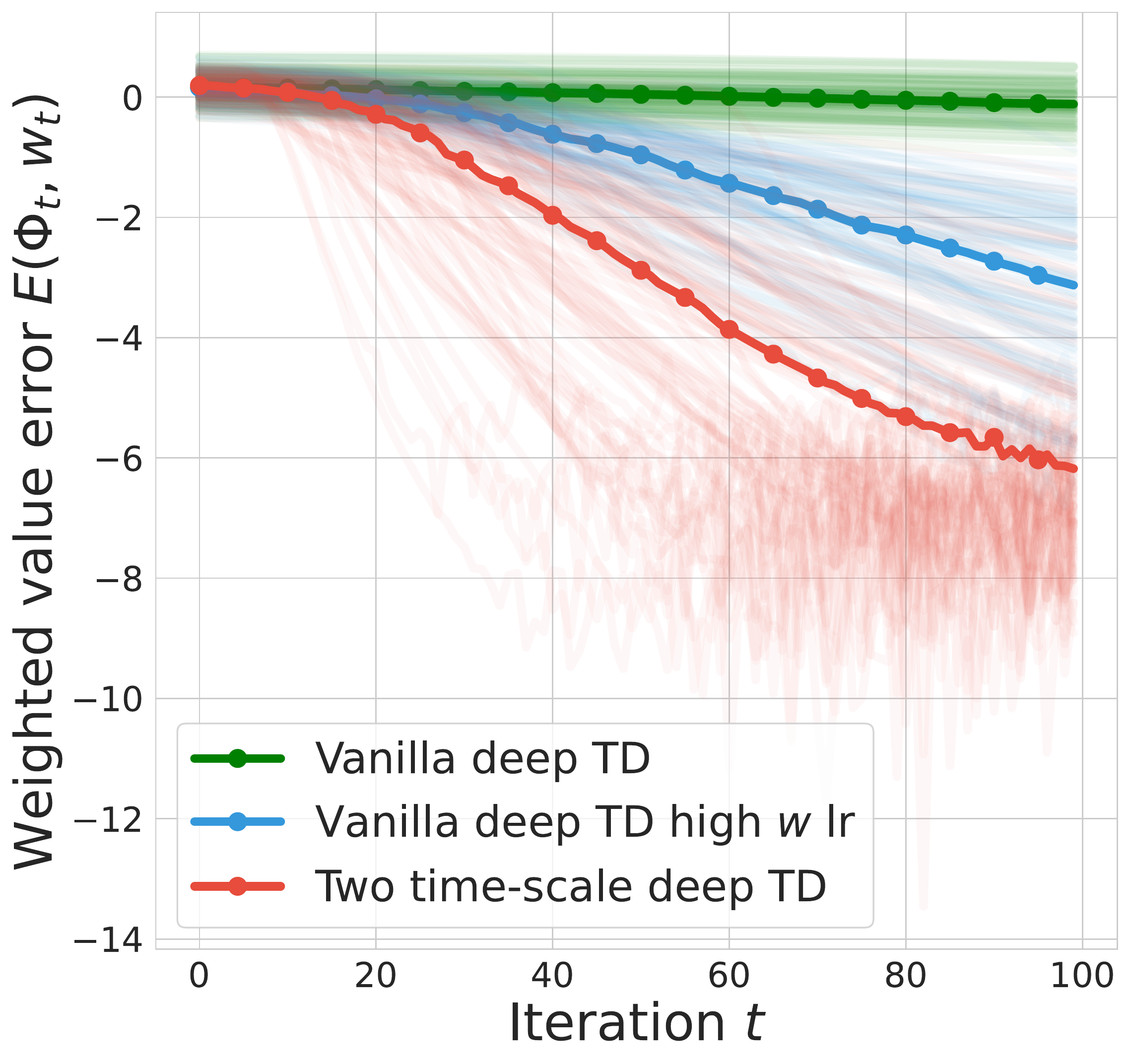}
     \caption{The evolution of the weighted value approximation error $E(\Phi_t,w_t)$ under the end-to-end linear TD dynamics. The solid curves show the median result over $100$ randomly generated MDPs. We compare three cases: (1) vanilla end-to-end linear TD with $\eta_\Phi=\eta_w=1$ (\cref{eq:deep-linear-td-ode}); (2) vanilla end-to-end linear TD with $\eta_\Phi=1,\eta_w=10$; (3) two time-scale end-to-end linear TD (\cref{eq:deep-linear-td-ode-two-time-scale}). Across all dynamics, the weighted error takes a decaying trend though such decay is not necessarily monotonic for randomly sampled MDPs. Fixing the number of iterations, the two time-scale dynamics achieves the fastest rate of decay compared to vanilla end-to-end linear TD.}
    \label{fig:weighted-errors}
\end{figure}

We expand on the result a bit more. The second term in the time derivative $\left\lVert \frac{dw_t}{dt} \right\rVert_2^2$ can be attributed to the classic linear TD update, whereas the first term $\left\lVert \frac{d\Phi_t}{dt} \right\rVert_2^2$ is derived from the updates on the representation matrix. Such a \emph{local improvement} property is indicative of the usefulness of end-to-end linear TD in practice, similar to how policy gradient updates provide local value improvements in policy optimization. The rate of the improvement in $E(\Phi_t,w_t)$ is also proportional to the magnitude of the semi-gradient updates to $\Phi_t$ and $w_t$. 

\paragraph{Critical points of end-to-end linear TD.} \cref{theoremvalue} implies the end-to-end linear TD dynamics should converge to a critical point, i.e., points at which $\dot{\Phi}_t=0$ and $\frac{dw_t}{dt}=0$, which we discuss below. 
Solving for $\frac{dw_t}{dt}=0$, we see that given any $\Phi$, the weight parameter must be at the TD fixed point $w=w_{\Phi}^\ast$. Then when solving for $\dot{\Phi}=0$, we obtain the following characterization for the critical points.
\begin{restatable}{lemma}{lemmacritical}\label{lemmacritical} Assume the matrix $\Phi_t^T D^\pi(I-\gamma P^\pi) \Phi_t$ is invertible for all time, then the set of critical points $\mathcal{C}$ of the end-to-end linear TD dynamics (\cref{eq:deep-linear-td-ode}) are pairs of $(\Phi,w)\in\mathcal{C}$ such that $w=w_\Phi^\ast$, and $\Phi$ satisfies the following condition,
\begin{align}
    D^\pi R^\pi (D^\pi R^\pi)^T \Phi \in \text{span}\left(D^\pi(I-\gamma P^\pi) \Phi\right).\label{eq:condition-critical}
\end{align}
Here, $\text{span}(A)$ denotes the vector subspace spanned by the columns of $A$.  The notation $A\in\text{span}(B)$ means that each column of matrix $A$ is in the subspace $\text{span}(B)$.
\end{restatable}

We note that in general, it is challenging to deliver a more concise and intuitive description of the critical points beyond \cref{eq:condition-critical}, however, with further assumptions we can provide more meaningful characterizations on the representations at critical points. See \cref{sec:spectral} for further discussions. 

\paragraph{Reversibility assumption.} The reversibility assumption on the Markov chain is a  sufficient condition in showing that the weighed error $E(\Phi_t,w_t)$ strictly decays over time (\cref{theoremvalue}), similar to the linear TD case \citep{ollivier2018approximate}. Lately, \citet{brandfonbrener2019geometric} also showed the convergence of non-linear TD dynamics when the MDP is reversible enough. All such prior results allude to the important impact of such an assumption. Nevertheless, it is a very strong assumption in general. We complement the theoretical results here with numerical simulations of the weighted error $E(\Phi_t,w_t)$ along the end-to-end linear TD dynamics, over general MDPs that can violate the reversibility assumption.

In \cref{fig:weighted-errors}, we show the evolution of $E(\Phi_t,w_t)$ over time under three different dynamics: (1) vanilla end-to-end linear TD with $\eta_\Phi=\eta_w=1$ (\cref{eq:deep-linear-td-ode}); (2) vanilla end-to-end linear TD with $\eta_\Phi=1,\eta_w=10$; (3) two time-scale end-to-end linear TD (\cref{eq:deep-linear-td-ode-two-time-scale}, see \cref{sec:spectral} for details). For each dynamics, we generate $100$ random MDPs that very likely violate the reversibility assumption. In all cases, the error takes a decaying trend though the decrease is not necessarily monotonic in general. This empirically shows that to certain degree, the general trend of \emph{decreasing value error} is robust to the violation of the assumption. Furthermore, we see that when the weight parameter is updated at a higher learning rate than the representation matrix, the error decays at a faster rate. The fastest rate seems to be obtained at the extreme when at each iteration, the weight parameter is set at the TD fixed point $w_t=w_{\Phi_t}^\ast$ under the limiting two time-scale dynamics (see \cref{sec:spectral}). See Appendix~\ref{appendix:exp} for further experimental details.

\paragraph{Convergence of the error and representation.}
Since the error function $E(\Phi_t,w_t)$ is strictly monotonic over time, it converges to an asymptotic value as $t\rightarrow\infty$. An important question is whether the error converges to the lowest possible value over the set of critical points $\lim_{t\rightarrow\infty}E(\Phi_t,w_t)=\inf_{(\Phi,w)\in\mathcal{C}}E(\Phi,w)$.  Answering such questions requires a refined analysis of the dynamics, which we leave to future work. A related question is what does the representation $\Phi_t$ converge to, if it converges at all. We provide some discussions on this question in the next section.

\section{Characterizing representations of end-to-end linear TD}
\label{sec:spectral}

Thus far we have characterized the behavior of the end-to-end linear TD dynamics (in \cref{eq:deep-linear-td-ode}) via the value approximation error $E(\Phi,w)$. Such an error characterizes the information that the aggregate prediction $\Phi w$ contains. In applications such as policy evaluation, where value prediction is the ultimate objective, such a characterization seems to suffice. 

In practice, it is also of interest to understand the information that the representation $\Phi$ contains by itself. Consider a motivating example (taken from \citep{bellemare2019geometric}) where $k=h=1$, if we take the representation to be the value function $\Phi=V^\pi$ and $w=1$, the prediction error is zero. However, such a representation is specialized to reward function $R^\pi$ and does not capture general information about the transition matrix $P^\pi$. 

\subsection{Warming up: two time-scale end-to-end linear TD}
To facilitate the discussion, we consider a specialized case of the end-to-end linear TD dynamics, where the weight parameter $w_t$ is updated at a much faster time pace than $\Phi_t$. Numerically, this can be emulated by the end-to-end linear TD updates in \cref{eq:deep-linear-td} with $\eta_\Phi\ll\eta_w$ as mentioned earlier. Here, we consider the extreme case, the \emph{two time-scale} dynamics where at at any iteration $t$, the weight parameter $w_t=w_{\Phi_t}^\ast$ is the optimal weight adapted to representation $\Phi_t$. 
\begin{align}
\begin{split}\label{eq:deep-linear-td-two-time-scale}
    w_t = w_{\Phi_t}^\ast,\  
    \Phi_{t+1} = \Phi_t - \eta_\Phi \partial_{\Phi_t} L(\Phi_t,w_t) 
\end{split}
\end{align}
We call the above dynamics the end-to-end linear TD-learning with two time-scales. The continuous time ODE system is as follows.
\begin{restatable}{lemma}{lemmadeeplineartwotimescaleode}\label{lemmadeeplineartwotimescaleode} The ODE to the two time-scale dynamics in \cref{eq:deep-linear-td-ode-two-time-scale} is 
\begin{align}
    \frac{d\Phi_t}{dt} = \eta_\Phi\cdot D^\pi\left(R^\pi-(I-\gamma P^\pi)\Phi_tw_{\Phi_t}^\ast\right)(w_{\Phi_t}^\ast)^T,\label{eq:deep-linear-td-ode-two-time-scale}
\end{align}
where recall that $w_{\Phi_t}^\ast$ is the linear TD fixed point to representation $\Phi_t$ specified in \cref{eq:linear-td-solution}. 
\end{restatable}
Importantly, note that the two time-scale ODE dynamics cannot be recovered exactly from the ODE to the vanilla end-to-end linear TD  \cref{eq:deep-linear-td-ode} with finite values of $\eta_w$ and $\eta_\Phi$.
Nevertheless, as a corollary to \cref{theoremvalue}, we can still show that under the reversibility assumption, the two time-scale dynamics also decreases the weighted value approximation error.
\begin{restatable}{corollary}{corovalue}\label{coroalue} Assume the Markov chain is reversible, then under the two time-scale learning dynamics in \cref{eq:deep-linear-td-ode-two-time-scale}, the weighted value approximation error is non-increasing $
    \frac{d}{dt}E(\Phi_t,w_t)\leq 0$.
If $\Phi_t$ is not at a critical point of the learning dynamics, then $ \frac{d}{dt}E(\Phi_t,w_t)$ is strictly negative. 
\end{restatable}

\subsection{Non-collapse representation dynamics}

One primary motivation to consider the two time-scale dynamics is that the representation matrix $\Phi_t$ enjoys the following important property: throughout the dynamics, the covariance matrix $\Phi_t^T\Phi_t$ is a constant matrix over time.
\begin{restatable}{lemma}{lemmaconstant}\label{lemmaconstant} Under the two time-scale dynamics in \cref{eq:deep-linear-td-ode-two-time-scale}, the covariance matrix $\Phi_t^T\Phi_t\in\mathbb{R}^{k\times k}$ is a constant matrix over time. 
\end{restatable}

\cref{lemmaconstant} implies that the representation matrix $\Phi_t$ maintains its representational capacity over time. Denoting the $i$-th column of $\Phi_t$ as $\Phi_{t,i}$ for $1\leq i\leq k$, we can visualize the evolution of $\Phi_t$ as \emph{rotations} in the representation space: the relative angle between $\Phi_{t,i}$ and $\Phi_{t,j}$, as well as their respective lengths, are preserved over time. For example, if the $k$ columns of $\Phi_t$ are initialized to be orthornormal $\Phi_0^T\Phi_0=I_{k\times k}$, then they stay orthonormal throughout the dynamics. This excludes situations where all $k$ columns converge to the same direction (see, e.g., dynamics in \citep{lyle2021effect} for examples), in which case the learned representation is not ideal.

\subsection{When does end-to-end linear TD learn useful representations}

In classic linear TD, the representations are fixed throughout the weight update. Examples of good \emph{fixed} representations include top eigen or singular vector decompositions of the transition matrix $P^\pi$
\citep{mahadevan2005proto,c.2018eigenoption,behzadian2019fast,ren2022spectral}, some of which can lead to provably low value approximation errors \citep{behzadian2019fast}.

Now, we seek to identify cases where end-to-end linear TD entails learning such useful representations. 
To measure the information contained in $\Phi_t$ about the transition matrix $P^\pi$, we define the trace objective
\begin{align}
    f(\Phi_t) \coloneqq  \text{Tr}\left(\Phi_t^T  (I-\gamma P^\pi)^{-1} \Phi_t\right)\label{eq:symmetric-trace-obj}.
\end{align}
The resolvant matrix $(I-\gamma P^\pi)^{-1}=\sum_{t=0}^\infty (\gamma P^\pi)^t$ reflects the long term discounted transition under $P^\pi$. To see why the above trace objective provides a helpful measure, note that when $P^\pi$ is symmetric, we can relate it to the data covariance matrix in PCA \citep{abdi2010principal}. Indeed, if we constrain inputs to $f$ to be the matrix $U$ concatenating any of $k$ distinct eigenvectors $(u_i)_{i=1}^k$ of $(I-\gamma P^\pi)^{-1}$, then $f(U)=\sum_{i=1}^k \lambda_i$ where $\lambda_i$ is the eigenvalue of $u_i$. As a result, solving the constrained optimization problem
\begin{align}
    \max_\Phi f(\Phi),\ \text{s.t.}\  \Phi^T\Phi=I_{k\times k}\label{eq:k-pca}
\end{align}
amounts to finding the top $k$ eigenvectors of $(I-\gamma P^\pi)^{-1}$.

Can we show that the end-to-end linear TD dynamics makes progress towards maximizing $f(\Phi)$?
Unfortunately, even under the restrictive setting that $P^\pi$ is symmetric, there is no guarantee that $f(\Phi_t)$ is maximized following the end-to-end linear TD dynamics. An probably intuitive explanation is that since $R^\pi$ comes into play in the TD updates, it may prevent $\Phi_t$ from learning just about the transition matrix (see numerical examples in \cref{subsec:random}). However, by imposing further conditions on the reward function, we show below that $\Phi_t$ learns to capture the spectral information of $P^\pi$.

\begin{restatable}{theorem}{theoremspectral}\label{theoremspectral} Assume the outer product of the reward function is an identity matrix $R^\pi(R^\pi)^T=I_{|\mathcal{X}|\times |\mathcal{X}|}$ and assume $P^\pi$ is symmetric. Then under the two time-scale dynamics in \cref{eq:deep-linear-td-ode-two-time-scale}, the trace objective is non-increasing $df/dt\geq 0$ over time. If $\frac{d\Phi_t}{dt}\neq 0$, then $df/dt>0$.
\end{restatable}

We will discuss the implication of   $R^\pi(R^\pi)^T=I$ in \cref{subsec:random}. Before, that, we note a critical observation: by \cref{lemmaconstant} and \cref{theoremspectral}, we see that the two time-scale end-to-end linear TD dynamics carries out gradient-based optimization to solve the constrained $k$-PCA problem (\cref{eq:k-pca}). A key requirement is to satisfy the orthonormal constraints. Fortunately, as \cref{lemmaconstant} suggests, if we initialize $\Phi_0$ to be orthonormal, the constraints $\Phi_t^T\Phi_t=I$ is implicitly satisfied throughout the dynamics.

\paragraph{Critical points span eigenspaces of $P^\pi$.} When assumptions in \cref{theoremspectral} are imposed, we can strengthen our characterizations of the critical points of the end-to-end linear TD dynamics, by relating them to
eigen subspaces of $P^\pi$. Here, we limit our attention to \emph{non-trivial} critical points that satisfy the constraints $\Phi^T\Phi=\Phi_0^T\Phi_0$ as other critical points cannot be arrived at from the learning dynamics (\cref{lemmaconstant}).

\begin{restatable}{corollary}{corocritical}\label{corocritical} When assumptions in \cref{theoremspectral} are satisfied, and further, assume $\Phi_0^T D^\pi(I-\gamma P^\pi) \Phi_0$ is invertible. A representation $\Phi$ is a non-trivial critical point of the two time-scale dynamics (\cref{eq:deep-linear-td-ode-two-time-scale}) if and only if it spans an invariant subspace of $P^\pi$, i.e., $
   P^\pi \Phi \in \text{span} \left(\Phi\right)$.
\end{restatable}

\subsection{Fitting values with randomly sampled rewards}\label{subsec:random}
Recall that $R^\pi\in\mathbb{R}^{|\mathcal{X}|\times h}$ can be understood as the concatenation of $h$ different reward functions, and we let $R_i^\pi\in\mathbb{R}^{|\mathcal{X}|}$ denote its $i$-th column for $1\leq i\leq h$. 
An especially interesting case is when entries of each one of the $h$ such reward columns are sampled i.i.d. 

\begin{restatable}{lemma}{lemmaexample}\label{lemmaexample} (\textbf{Randomly sampled rewards}) When the each entry of $R_{i}^\pi$ is sampled i.i.d. from any zero-mean distribution with finite variance $\sigma^2/\sqrt{h}$ for some constant $\sigma>0$, for  $1\leq i\leq h$. Then as the number of reward columns $h\rightarrow\infty$, we have
$
    R^\pi(R^\pi)^T \rightarrow \sigma^2 I_{|\mathcal{X}|\times |\mathcal{X}|}$ almost surely. 
\end{restatable}

\begin{figure}[t]
    \centering
    \includegraphics[keepaspectratio,width=.45\textwidth]{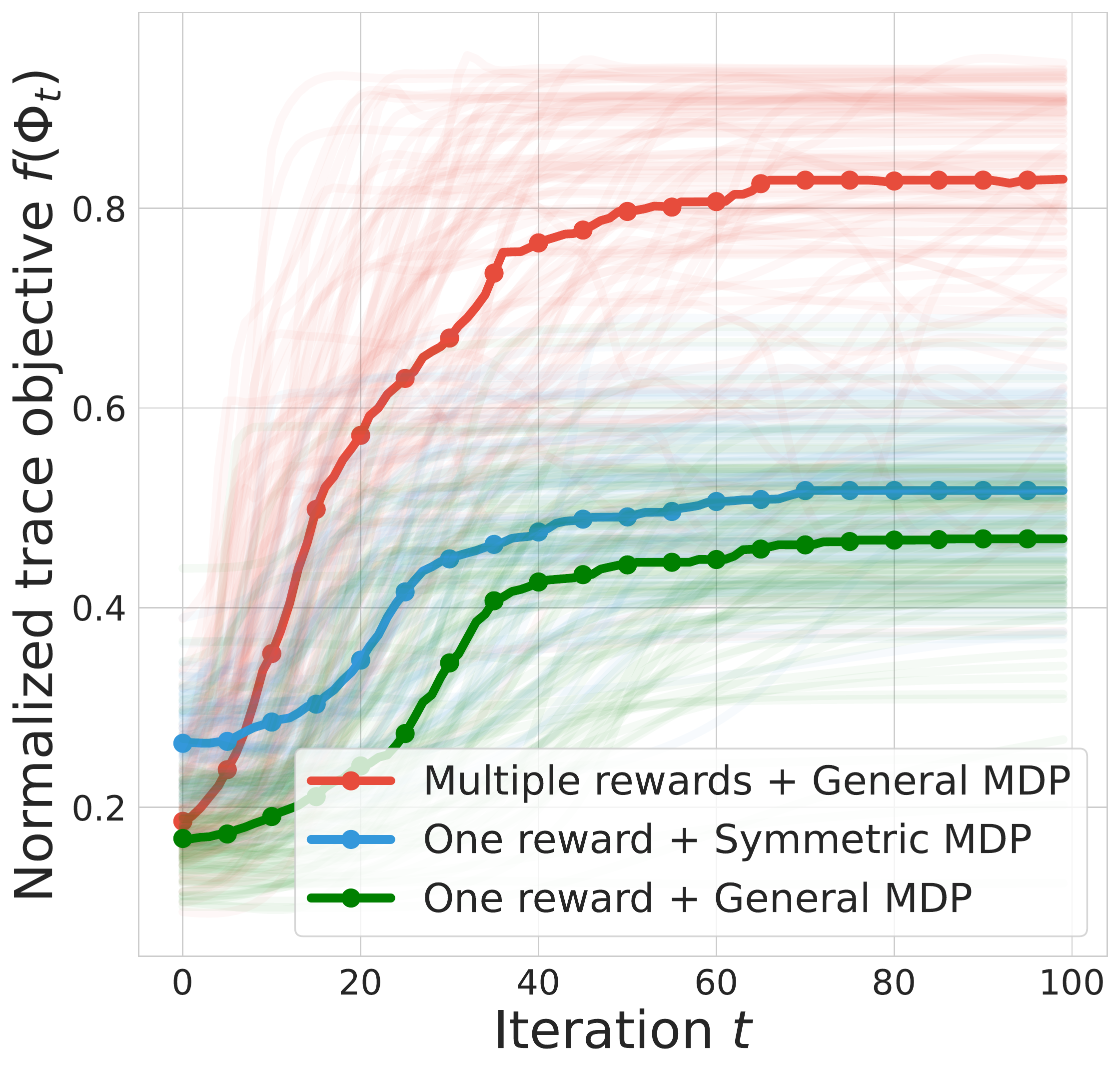}
     \caption{The normalized trace objective $f(\Phi_t)$ under the end-to-end linear TD dynamics. We generate $100$ random MDPs and solid curves show the median results. We consider three scenarios: (1) multiple reward functions $h=5$ with general $P^\pi$; (2) single reward $h=1$ with symmetric $P^\pi$; (3) single reward with general $P^\pi$. The trace objective takes an upward trend though the improvement is not necessarily monotonic for randomly sampled MDPs. When predicting multiple random rewards, the representations seem to capture more spectral information about $P^\pi$ as measured by the trace objective.}
    \label{fig:trace}
\end{figure}


\begin{figure}[t]
    \centering
    \includegraphics[keepaspectratio,width=.45\textwidth]{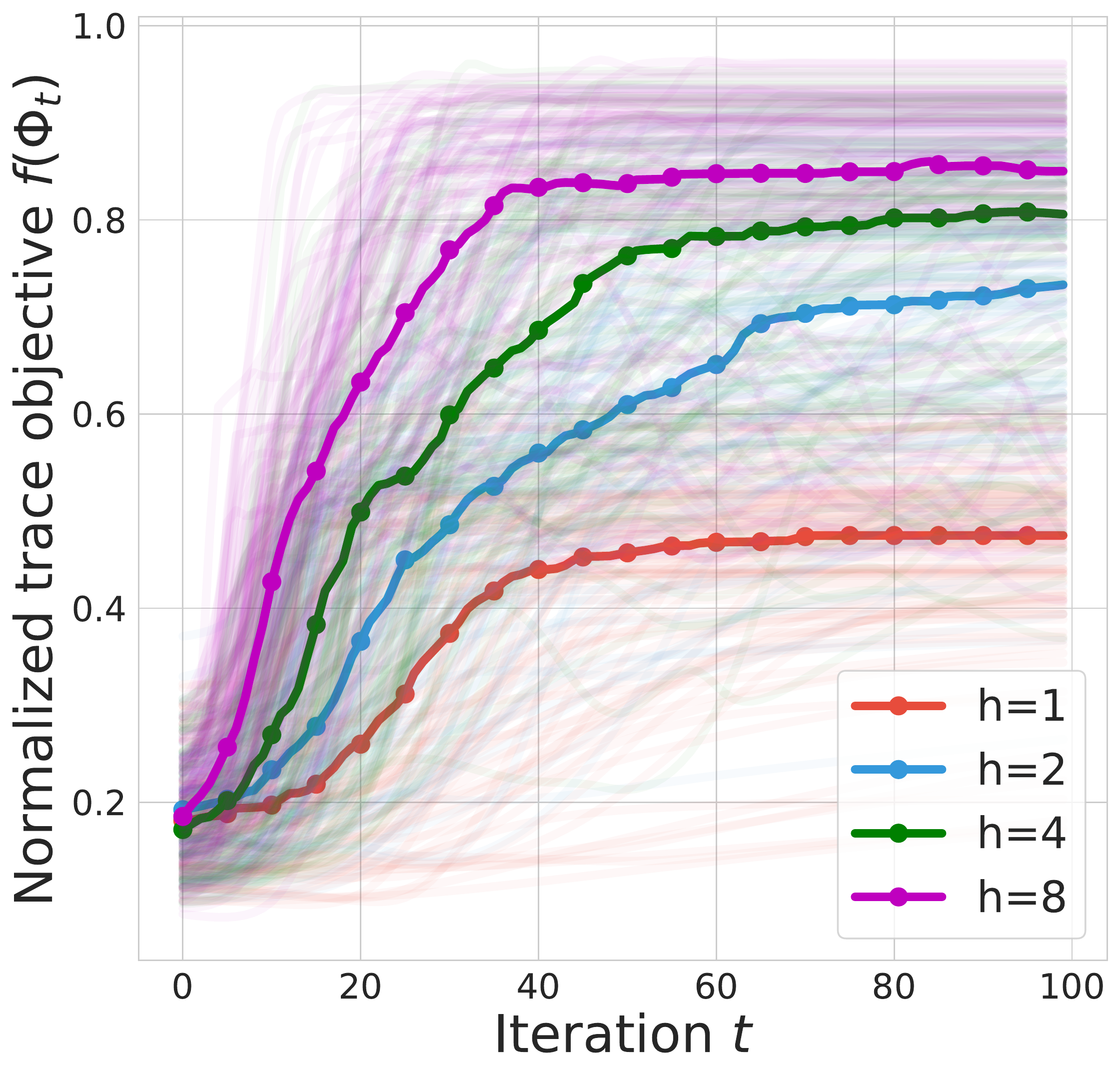}
     \caption{The normalized trace objective $f(\Phi_t)$ under the end-to-end linear TD dynamics. We generate $100$ random MDPs with general $P^\pi$ and solid curves show the median results. We study the effect of the number of random rewards $h$, with $h\in\{1,2,4,8\}$. As $h$ increases, the median trace objective $f(\Phi_t)$ improves both at higher rates and achieves higher asymptotic values, indicating that the representations converge to the top $k$ eigensubspace of $P^\pi$. In this case, the state space $|\mathcal{X}|=30$ and $h=4$ seems to suffice for learning very good representations.}
    \label{fig:trace-rewards}
\end{figure}

Now, it is useful to recall \cref{theoremvalue} which showed that the weighted value errors decay over time under the end-to-end linear TD dynamics. In light of this previous result, we can interpret \cref{theoremspectral} and its connections to randomly sampled rewards as follows: if starting from a representation $\Phi_t$, one can accurately fit value functions $V_i^\pi$ generated by multiple random rewards $R_i^\pi$ simultaneously, then $\Phi_t$ must capture spectral information about the transition matrix $P^\pi$. 

\paragraph{Numerical experiments on randomly sampled MDPs.} In practice, the symmetry assumption in \cref{theoremspectral} is likely to be violated. In \cref{fig:trace} we evaluate the trace objective $f(\Phi_t)$ under the two time-scale dynamics with randomly generated MDPs. We compare a few cases: (1) multiple reward functions $h=5$ with general $P^\pi$; (2) single reward $h=1$ with symmetric $P^\pi$; (3) single reward with general $P^\pi$. The results suggest that the upward trend of the trace objective is fairly persistent even in asymmetric MDPs, implying that the dynamics entails $\Phi_t$ to capture useful information about $P^\pi$ in general. However, the improvement in $f$ is not monotonic in general. 

Another critical observation is that by fitting more random reward functions at the same time (red $h=5$ vs. green $h=1$), the dynamics seemingly captures more spectral information about the transition. The result also implies that in practice, we probably do not need a very large number $h$ of rewards, see \cref{sec:exp} for an ablation experiment.

\paragraph{How many reward functions are enough?} The above arguments imply that we might need a large number of reward functions for the condition $R^\pi(R^\pi)^T=I$ to hold. In practice, a key empirical question is how many random reward functions suffice for the end-to-end linear TD dynamics to learn good representations. We carry out an ablation study in \cref{fig:trace-rewards} where we consider $h\in\{1,2,4,8\}$. As $h$ increases, the trace objective $f(\Phi_t)$ both improves at a higher average rate, and achieves higher asymptotic values. In our tabular experiments, we consider $k=2$ dimensions for the representations. Intriguingly, $h=8$ seems to suffice to learning good representations (when the normalized trace objective $\geq 0.8$; note that if $P^\pi$ is symmetric, the maximum normalized trace objective is $1$), which is noticeably smaller than the number of states $|\mathcal{X}|=30$. Though the conclusion certainly depends on the specific MDPs, the result suggests in practice, potentially a relatively small number of random rewards is enough to induce good representations.

\paragraph{Connections to auxiliary tasks in prior work.} The idea of fitting value functions of randomly generated rewards bears close connections to a number of prior approach for building auxiliary tasks in deep RL, such as random cumulant predictions \citep{dabney2021value,zheng2021learning}. Random cumulant predictions consist in learning $h$ value functions starting from a single representation $\Phi_t$, with different value \emph{head}s. In our terminology, such value heads can be understood as separate weight parameter $w_i\in\mathbb{R}^k,1\leq i\leq h$, each dedicated to learning value function $V_i^\pi\in\mathbb{R}^{|\mathcal{X}|}$ with randomly generated reward $R_i^\pi$. This is mathematically equivalent to our matrix notation with $w\in\mathbb{R}^{k\times h}$ throughout the paper.

\section{Discussion on related work} \label{sec:discussion}

A more comprehensive discussion is in \cref{appendix:related}.

\paragraph{From linear TD to end-to-end linear TD.} Much of the prior work has focused on the linear TD setup \citep{sutton1998,tsitsiklis1996analysis}, where the representations are assumed fixed throughout learning.
Closely related to our work is \citep{lyle2021effect} where they proposed to understand the learning dynamics of end-to-end linear TD through its corresponding ODE system. Since such an ODE system is highly non-linear, it is more challenging to provide generic characterizations without restrictive assumptions. \citet{lyle2021effect} bypasses the non-linearity issue by essentially assuming a fixed weight parameter $w_t\equiv w$. In light of \cref{eq:deep-linear-td-ode}, they study the following dynamics
\begin{align*}
    \frac{d\Phi_t}{dt} = \eta_\Phi \cdot D^\pi\left(R^\pi -(I-\gamma P^\pi)\Phi_t w\right)w^T
\end{align*}
which effectively reduces  to a linear system in $\Phi_t$ and is more amenable to analysis.

Our work differs in a few important aspects. Firstly, our analysis adheres strictly to the vanilla end-to-end linear TD  (\cref{eq:deep-linear-td-ode}) or two-time scale end-to-end linear TD  (\cref{eq:deep-linear-td-ode-two-time-scale}), without imposing the constant weight assumption as in \citep{lyle2019comparative}. Our analysis is also slightly more general, as it consists in constructing a few scalar functions that characterize the dynamics. This is more applicable to general ODEs where obtaining exact solutions is not tractable.

\paragraph{TD with non-linear function approximations.} Along the line of research on TD-dynamics with non-linear function approximation, a closely related work is \citep{brandfonbrener2019geometric} where they established the convergence behavior of TD-learning with smooth homogeneous functions. A key requirement underlying their result is that the MDP is sufficiently reversible, which echos the assumption we make in \cref{theoremvalue}. However, under their framework there is no clear notion of representation as defined in the bi-linear case.

\begin{figure}[t]
    \centering
    \includegraphics[keepaspectratio,width=.45\textwidth]{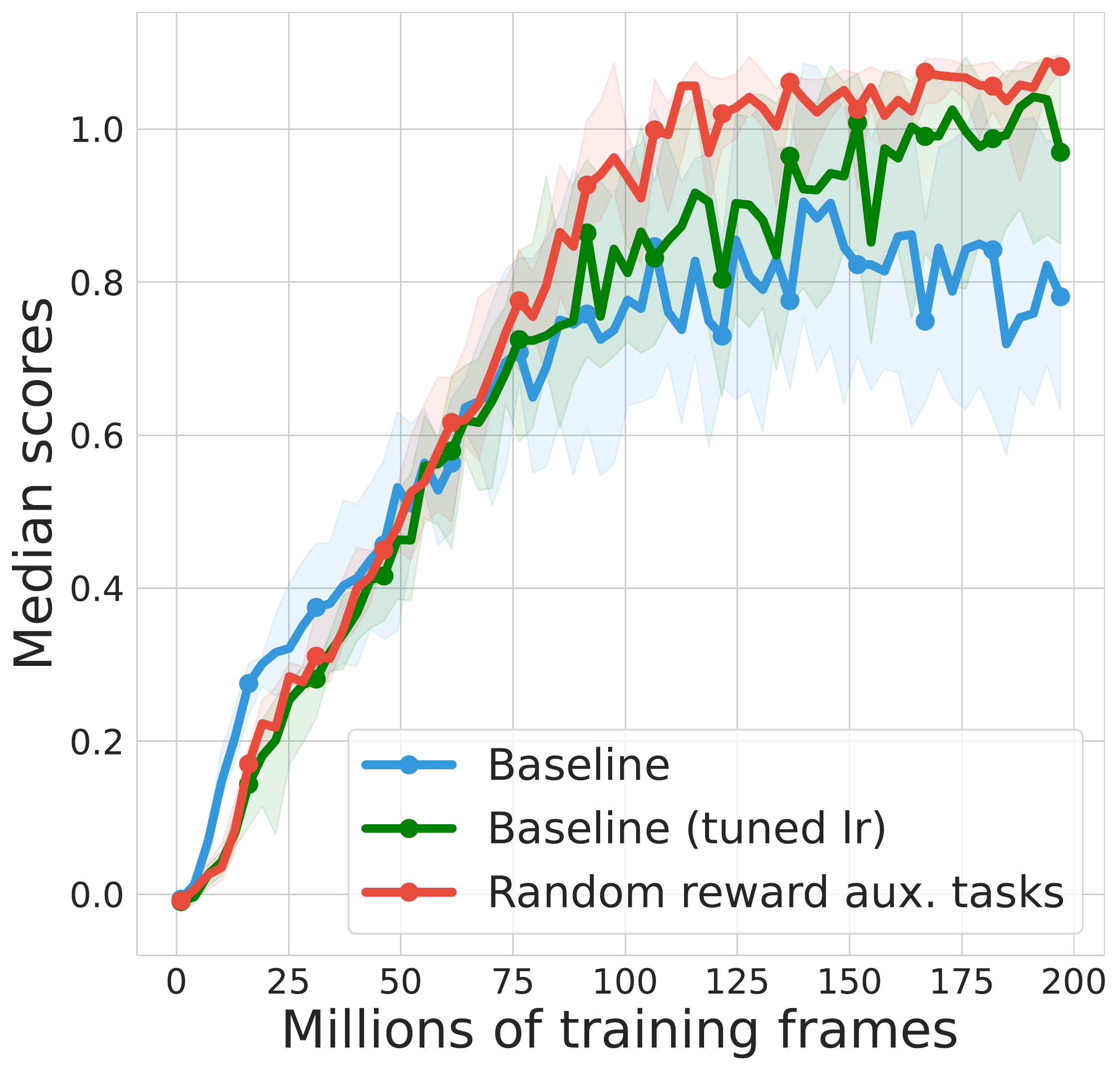}
     \caption{Comparison of training curves of median human normalized scores across 15 Atari games. We compare the baseline DQN, DQN with tuned learning rate and DQN with random value predictions as auxiliary tasks. The shaded areas show the $95\%$ bootstrapped confidence intervals averaged over $3$ seeds. The random value prediction tasks seem to provide marginal benefits over the tuned DQN. See \cref{appendix:exp} for further details.}
    \label{fig:atari-median}
\end{figure}

\paragraph{Self-predictive learning.} Motivated by value-based TD-learning, self-predictive learning directly employs the notion of bootstrapping to the representation space, which has produced a number of empirically successful implementations \citep{guo2020bootstrap,schwarzer2021dataefficient,guo2020bootstrap}. Recently,  \citet{tang2022understanding} proposes to understand the behavior of self-predictive learning through its corresponding ODE. From their discussion, we note that two time-scale dynamics is a generic way to enforcing optimization constraints. Our results in \cref{sec:spectral} build on this and show a similar characterization for the end-to-end linear TD.

\paragraph{Two time-scale learning.} \citet{levine2017shallow} proposed a combination of LSTD updates for the linear weights on top of DQN representations, which are shaped by gradient descents. In continuous time, this is effectively the two time-scale dynamics (Eqn~\eqref{eq:deep-linear-td-ode-two-time-scale}). While their discussion is empirically motivated, we focus on the representation learning aspect of such a learning dynamics.

\paragraph{Auxiliary tasks.} In deep RL literature, it has been empirically observed that certain extra training objectives, normally referred to as \emph{auxiliary tasks} (see, e.g., \citep{beattie2016deepmind,jaderberg2016reinforcement,bellemare2019geometric,fedus2019hyperbolic,dabney2021value,zheng2021learning}). Though the conventional wisdom is that such auxiliary tasks are useful in shaping representations, it is generally difficult to characterize exactly what representations are induced, despite some recent theoretical efforts \citep{lyle2021effect}. Our results formally demonstrate the potential benefits of random reward predictions when combined with TD-learning.

\section{Experiments}
\label{sec:exp}

We validate the potential usefulness of fitting randomly generated value functions as an auxiliary task in deep RL settings. We use DQN \citep{mnih2013} as a baseline, and generate random reward functions $R_i^\pi(x,a)$ via outputs of randomly initialized networks, following the practice of \citep{dabney2021value}. The common deep RL equivalent of the representation $\phi_x$ is the output of the torso network (usually a convnet). For the auxiliary task, we fit multiple value function by applying different head network (usually MLPs) on top of the common output $\phi_x$. Throughout, we update all network parameters with the same learning rate. We have ablated on increasing the learning rate for head networks, so that the algorithm is more in line with the two time-scale dynamics; however, we found that in practice, this tends to provide the torso network with less learning signal (as much is taken care of by the head network) and usually impedes performance. Understanding the theory-practice gap here would be an important future direction.

\cref{fig:atari-median} shows the result for DQN, DQN with tuned learning rate, and DQN with random value function auxiliary task. Overall, we see that the auxiliary task provides marginal improvement on the tuned DQN baseline, hinting at the potential usefulness of the auxiliary objective for representation learning. See \cref{appendix:exp} for further experimental details.

\section{Conclusion}
We have provided a number of theoretical analysis on the representation dynamics under end-to-end end-to-end linear TD dynamics. Under the reversibility assumption, we have showed that the representation evolves such that the value approximation error decreases over time. With further assumptions on the reward function, we establish that the representation dynamics carries out spectral decomposition on the transition matrix, leading to provably useful representations. This further implies that fitting random value functions is a principled auxiliary task for representation learning. 

Our work opens up a few interesting directions for further investigation, such as how to relax the reversibility assumptions, how to combine the end-to-end linear TD framework with nonlinear functions and extensions to the control case.

\bibliography{main}

\begin{thebibliography}{37}
\providecommand{\natexlab}[1]{#1}
\providecommand{\url}[1]{\texttt{#1}}
\expandafter\ifx\csname urlstyle\endcsname\relax
  \providecommand{\doi}[1]{doi: #1}\else
  \providecommand{\doi}{doi: \begingroup \urlstyle{rm}\Url}\fi

\bibitem[Abdi and Williams(2010)]{abdi2010principal}
Herv{\'e} Abdi and Lynne~J Williams.
\newblock Principal component analysis.
\newblock \emph{Wiley interdisciplinary reviews: computational statistics},
  2\penalty0 (4):\penalty0 433--459, 2010.

\bibitem[Agarwal et~al.(2021)Agarwal, Schwarzer, Castro, Courville, and
  Bellemare]{agarwal2021deep}
Rishabh Agarwal, Max Schwarzer, Pablo~Samuel Castro, Aaron~C Courville, and
  Marc Bellemare.
\newblock Deep reinforcement learning at the edge of the statistical precipice.
\newblock \emph{Advances in neural information processing systems},
  34:\penalty0 29304--29320, 2021.

\bibitem[Agazzi and Lu(2022)]{agazzi2022temporal}
Andrea Agazzi and Jianfeng Lu.
\newblock Temporal-difference learning with nonlinear function approximation:
  lazy training and mean field regimes.
\newblock In \emph{Mathematical and Scientific Machine Learning}, pages 37--74.
  PMLR, 2022.

\bibitem[Beattie et~al.(2016)Beattie, Leibo, Teplyashin, Ward, Wainwright,
  K{\"u}ttler, Lefrancq, Green, Vald{\'e}s, Sadik, et~al.]{beattie2016deepmind}
Charles Beattie, Joel~Z Leibo, Denis Teplyashin, Tom Ward, Marcus Wainwright,
  Heinrich K{\"u}ttler, Andrew Lefrancq, Simon Green, V{\'\i}ctor Vald{\'e}s,
  Amir Sadik, et~al.
\newblock Deepmind lab.
\newblock \emph{arXiv preprint arXiv:1612.03801}, 2016.

\bibitem[Behzadian et~al.(2019)Behzadian, Gharatappeh, and
  Petrik]{behzadian2019fast}
Bahram Behzadian, Soheil Gharatappeh, and Marek Petrik.
\newblock Fast feature selection for linear value function approximation.
\newblock In \emph{Proceedings of the International Conference on Automated
  Planning and Scheduling}, volume~29, pages 601--609, 2019.

\bibitem[Bellemare et~al.(2019)Bellemare, Dabney, Dadashi, Ali~Taiga, Castro,
  Le~Roux, Schuurmans, Lattimore, and Lyle]{bellemare2019geometric}
Marc Bellemare, Will Dabney, Robert Dadashi, Adrien Ali~Taiga, Pablo~Samuel
  Castro, Nicolas Le~Roux, Dale Schuurmans, Tor Lattimore, and Clare Lyle.
\newblock A geometric perspective on optimal representations for reinforcement
  learning.
\newblock \emph{Advances in neural information processing systems}, 32, 2019.

\bibitem[Bellemare et~al.(2013)Bellemare, Naddaf, Veness, and
  Bowling]{bellemare2013arcade}
Marc~G. Bellemare, Yavar Naddaf, Joel Veness, and Michael Bowling.
\newblock The arcade learning environment: An evaluation platform for general
  agents.
\newblock \emph{Journal of Artificial Intelligence Research}, 47:\penalty0
  253--279, 2013.

\bibitem[Brandfonbrener and Bruna(2019)]{brandfonbrener2019geometric}
David Brandfonbrener and Joan Bruna.
\newblock Geometric insights into the convergence of nonlinear td learning.
\newblock \emph{arXiv preprint arXiv:1905.12185}, 2019.

\bibitem[Cai et~al.(2019)Cai, Yang, Lee, and Wang]{cai2019neural}
Qi~Cai, Zhuoran Yang, Jason~D Lee, and Zhaoran Wang.
\newblock Neural temporal-difference and q-learning provably converge to global
  optima.
\newblock \emph{arXiv preprint arXiv:1905.10027}, 2019.

\bibitem[Dabney et~al.(2021)Dabney, Barreto, Rowland, Dadashi, Quan, Bellemare,
  and Silver]{dabney2021value}
Will Dabney, Andr{\'e} Barreto, Mark Rowland, Robert Dadashi, John Quan, Marc~G
  Bellemare, and David Silver.
\newblock The value-improvement path: Towards better representations for
  reinforcement learning.
\newblock In \emph{Proceedings of the AAAI Conference on Artificial
  Intelligence}, volume~35, pages 7160--7168, 2021.

\bibitem[Fedus et~al.(2019)Fedus, Gelada, Bengio, Bellemare, and
  Larochelle]{fedus2019hyperbolic}
William Fedus, Carles Gelada, Yoshua Bengio, Marc~G Bellemare, and Hugo
  Larochelle.
\newblock Hyperbolic discounting and learning over multiple horizons.
\newblock \emph{arXiv preprint arXiv:1902.06865}, 2019.

\bibitem[Guo et~al.(2020)Guo, Pires, Piot, Grill, Altch{\'e}, Munos, and
  Azar]{guo2020bootstrap}
Zhaohan~Daniel Guo, Bernardo~Avila Pires, Bilal Piot, Jean-Bastien Grill,
  Florent Altch{\'e}, R{\'e}mi Munos, and Mohammad~Gheshlaghi Azar.
\newblock Bootstrap latent-predictive representations for multitask
  reinforcement learning.
\newblock In \emph{International Conference on Machine Learning}, pages
  3875--3886. PMLR, 2020.

\bibitem[Jaderberg et~al.(2016)Jaderberg, Mnih, Czarnecki, Schaul, Leibo,
  Silver, and Kavukcuoglu]{jaderberg2016reinforcement}
Max Jaderberg, Volodymyr Mnih, Wojciech~Marian Czarnecki, Tom Schaul, Joel~Z
  Leibo, David Silver, and Koray Kavukcuoglu.
\newblock Reinforcement learning with unsupervised auxiliary tasks.
\newblock \emph{arXiv preprint arXiv:1611.05397}, 2016.

\bibitem[Kingma and Ba(2015)]{kingma2014adam}
Diederik~P. Kingma and Jimmy Ba.
\newblock Adam: A method for stochastic optimization.
\newblock In \emph{Proceedings of the International Conference on Learning
  Representations}, 2015.

\bibitem[Levine et~al.(2017)Levine, Zahavy, Mankowitz, Tamar, and
  Mannor]{levine2017shallow}
Nir Levine, Tom Zahavy, Daniel~J Mankowitz, Aviv Tamar, and Shie Mannor.
\newblock Shallow updates for deep reinforcement learning.
\newblock \emph{Advances in Neural Information Processing Systems}, 30, 2017.

\bibitem[Lyle et~al.(2019)Lyle, Bellemare, and Castro]{lyle2019comparative}
Clare Lyle, Marc~G. Bellemare, and Pablo~Samuel Castro.
\newblock A comparative analysis of expected and distributional reinforcement
  learning.
\newblock In \emph{Proceedings of the AAAI Conference on Artificial
  Intelligence}, volume~33, pages 4504--4511, 2019.

\bibitem[Lyle et~al.(2021)Lyle, Rowland, Ostrovski, and Dabney]{lyle2021effect}
Clare Lyle, Mark Rowland, Georg Ostrovski, and Will Dabney.
\newblock On the effect of auxiliary tasks on representation dynamics.
\newblock In \emph{International Conference on Artificial Intelligence and
  Statistics}, pages 1--9. PMLR, 2021.

\bibitem[Machado et~al.(2018)Machado, Rosenbaum, Guo, Liu, Tesauro, and
  Campbell]{c.2018eigenoption}
Marlos~C. Machado, Clemens Rosenbaum, Xiaoxiao Guo, Miao Liu, Gerald Tesauro,
  and Murray Campbell.
\newblock Eigenoption discovery through the deep successor representation.
\newblock In \emph{International Conference on Learning Representations}, 2018.
\newblock URL \url{https://openreview.net/forum?id=Bk8ZcAxR-}.

\bibitem[Maei et~al.(2009)Maei, Szepesvari, Bhatnagar, Precup, Silver, and
  Sutton]{maei2009convergent}
Hamid Maei, Csaba Szepesvari, Shalabh Bhatnagar, Doina Precup, David Silver,
  and Richard~S Sutton.
\newblock Convergent temporal-difference learning with arbitrary smooth
  function approximation.
\newblock \emph{Advances in neural information processing systems}, 22, 2009.

\bibitem[Mahadevan(2005)]{mahadevan2005proto}
Sridhar Mahadevan.
\newblock Proto-value functions: Developmental reinforcement learning.
\newblock In \emph{Proceedings of the 22nd international conference on Machine
  learning}, pages 553--560, 2005.

\bibitem[Melo et~al.(2008)Melo, Meyn, and Ribeiro]{melo2008analysis}
Francisco~S Melo, Sean~P Meyn, and M~Isabel Ribeiro.
\newblock An analysis of reinforcement learning with function approximation.
\newblock In \emph{Proceedings of the 25th international conference on Machine
  learning}, pages 664--671, 2008.

\bibitem[Mnih et~al.(2013)Mnih, Kavukcuoglu, Silver, Graves, Antonoglou,
  Wierstra, and Riedmiller]{mnih2013}
Volodymyr Mnih, Koray Kavukcuoglu, David Silver, Alex Graves, Ioannis
  Antonoglou, Daan Wierstra, and Martin Riedmiller.
\newblock Playing atari with deep reinforcement learning.
\newblock \emph{arXiv preprint arXiv:1312.5602}, 2013.

\bibitem[Munos(2003)]{munos2003error}
R{\'e}mi Munos.
\newblock Error bounds for approximate policy iteration.
\newblock In \emph{ICML}, volume~3, pages 560--567. Citeseer, 2003.

\bibitem[Ollivier(2018)]{ollivier2018approximate}
Yann Ollivier.
\newblock Approximate temporal difference learning is a gradient descent for
  reversible policies.
\newblock \emph{arXiv preprint arXiv:1805.00869}, 2018.

\bibitem[Ren et~al.(2022)Ren, Zhang, Lee, Gonzalez, Schuurmans, and
  Dai]{ren2022spectral}
Tongzheng Ren, Tianjun Zhang, Lisa Lee, Joseph~E Gonzalez, Dale Schuurmans, and
  Bo~Dai.
\newblock Spectral decomposition representation for reinforcement learning.
\newblock \emph{arXiv preprint arXiv:2208.09515}, 2022.

\bibitem[Schaul et~al.(2015)Schaul, Quan, Antonoglou, and
  Silver]{schaul2015prioritized}
Tom Schaul, John Quan, Ioannis Antonoglou, and David Silver.
\newblock Prioritized experience replay.
\newblock \emph{arXiv preprint arXiv:1511.05952}, 2015.

\bibitem[Schwarzer et~al.(2021)Schwarzer, Anand, Goel, Hjelm, Courville, and
  Bachman]{schwarzer2021dataefficient}
Max Schwarzer, Ankesh Anand, Rishab Goel, R~Devon Hjelm, Aaron Courville, and
  Philip Bachman.
\newblock Data-efficient reinforcement learning with self-predictive
  representations.
\newblock In \emph{International Conference on Learning Representations}, 2021.
\newblock URL \url{https://openreview.net/forum?id=uCQfPZwRaUu}.

\bibitem[Sirignano and Spiliopoulos(2022)]{sirignano2022asymptotics}
Justin Sirignano and Konstantinos Spiliopoulos.
\newblock Asymptotics of reinforcement learning with neural networks.
\newblock \emph{Stochastic Systems}, 12\penalty0 (1):\penalty0 2--29, 2022.

\bibitem[Sutton(1988)]{sutton1988learning}
Richard~S. Sutton.
\newblock Learning to predict by the methods of temporal differences.
\newblock \emph{Machine learning}, 3\penalty0 (1):\penalty0 9--44, 1988.

\bibitem[Sutton and Barto(1998)]{sutton1998}
Richard~S. Sutton and Andrew~G. Barto.
\newblock \emph{Reinforcement Learning: An Introduction}.
\newblock MIT Press, 1998.

\bibitem[Sutton et~al.(2011)Sutton, Modayil, Delp, Degris, Pilarski, White, and
  Precup]{sutton2011horde}
Richard~S Sutton, Joseph Modayil, Michael Delp, Thomas Degris, Patrick~M
  Pilarski, Adam White, and Doina Precup.
\newblock Horde: A scalable real-time architecture for learning knowledge from
  unsupervised sensorimotor interaction.
\newblock In \emph{The 10th International Conference on Autonomous Agents and
  Multiagent Systems-Volume 2}, pages 761--768, 2011.

\bibitem[Sutton et~al.(2016)Sutton, Mahmood, and White]{sutton2016emphatic}
Richard~S Sutton, A~Rupam Mahmood, and Martha White.
\newblock An emphatic approach to the problem of off-policy temporal-difference
  learning.
\newblock \emph{The Journal of Machine Learning Research}, 17\penalty0
  (1):\penalty0 2603--2631, 2016.

\bibitem[Tang et~al.(2022)Tang, Guo, Richemond, Pires, Chandak, Munos, Rowland,
  Azar, Lan, Lyle, et~al.]{tang2022understanding}
Yunhao Tang, Zhaohan~Daniel Guo, Pierre~Harvey Richemond, Bernardo~{\'A}vila
  Pires, Yash Chandak, R{\'e}mi Munos, Mark Rowland, Mohammad~Gheshlaghi Azar,
  Charline~Le Lan, Clare Lyle, et~al.
\newblock Understanding self-predictive learning for reinforcement learning.
\newblock \emph{arXiv preprint arXiv:2212.03319}, 2022.

\bibitem[Tsitsiklis and Van~Roy(1996)]{tsitsiklis1996analysis}
John Tsitsiklis and Benjamin Van~Roy.
\newblock Analysis of temporal-diffference learning with function
  approximation.
\newblock \emph{Advances in neural information processing systems}, 9, 1996.

\bibitem[Virtanen et~al.(2020)Virtanen, Gommers, Oliphant, Haberland, Reddy,
  Cournapeau, Burovski, Peterson, Weckesser, Bright, {van der Walt}, Brett,
  Wilson, Millman, Mayorov, Nelson, Jones, Kern, Larson, Carey, Polat, Feng,
  Moore, {VanderPlas}, Laxalde, Perktold, Cimrman, Henriksen, Quintero, Harris,
  Archibald, Ribeiro, Pedregosa, {van Mulbregt}, and {SciPy 1.0
  Contributors}]{2020SciPy-NMeth}
Pauli Virtanen, Ralf Gommers, Travis~E. Oliphant, Matt Haberland, Tyler Reddy,
  David Cournapeau, Evgeni Burovski, Pearu Peterson, Warren Weckesser, Jonathan
  Bright, St{\'e}fan~J. {van der Walt}, Matthew Brett, Joshua Wilson, K.~Jarrod
  Millman, Nikolay Mayorov, Andrew R.~J. Nelson, Eric Jones, Robert Kern, Eric
  Larson, C~J Carey, {\.I}lhan Polat, Yu~Feng, Eric~W. Moore, Jake
  {VanderPlas}, Denis Laxalde, Josef Perktold, Robert Cimrman, Ian Henriksen,
  E.~A. Quintero, Charles~R. Harris, Anne~M. Archibald, Ant{\^o}nio~H. Ribeiro,
  Fabian Pedregosa, Paul {van Mulbregt}, and {SciPy 1.0 Contributors}.
\newblock {{SciPy} 1.0: Fundamental Algorithms for Scientific Computing in
  Python}.
\newblock \emph{Nature Methods}, 17:\penalty0 261--272, 2020.

\bibitem[Zhang et~al.(2020)Zhang, Cai, Yang, Chen, and Wang]{zhang2020can}
Yufeng Zhang, Qi~Cai, Zhuoran Yang, Yongxin Chen, and Zhaoran Wang.
\newblock Can temporal-difference and q-learning learn representation? a
  mean-field theory.
\newblock \emph{arXiv preprint arXiv:2006.04761}, 2020.

\bibitem[Zheng et~al.(2021)Zheng, Veeriah, Vuorio, Lewis, and
  Singh]{zheng2021learning}
Zeyu Zheng, Vivek Veeriah, Risto Vuorio, Richard~L Lewis, and Satinder Singh.
\newblock Learning state representations from random deep action-conditional
  predictions.
\newblock \emph{Advances in Neural Information Processing Systems},
  34:\penalty0 23679--23691, 2021.

\end{thebibliography}
\bibliographystyle{plainnat}

\clearpage
\onecolumn

\begin{appendix}

\section*{\centering APPENDICES: Towards a Better Understanding of Representation Dynamics under TD-learning}

\section{Derivation of continuous time ODEs for end-to-end linear TD}

Recall the squared prediction error in \cref{eq:sg-squared-error},
\begin{align*}
    L(\Phi,w) \coloneqq \left\lVert \text{sg}\left(R^\pi(x_t) + \gamma\phi_{x_{t+1}}^tw\right)-\phi_{x_t}^Tw\right\rVert_2^2.
\end{align*} 
We take an expectation over the state distribution $x_t\sim d^\pi, x_{t+1}\sim P^\pi(\cdot|x_t)$, the expected loss function can be expressed as follows
\begin{align*}
    \mathbb{E}[L] = \sum_{x} d^\pi(x) \left(\text{sg}\left(R^\pi(x) + \gamma \left(P^\pi \Phi w\right) (x)\right) - \left(\Phi w\right)(x) \right)^2,
\end{align*}
where $\left(\Phi w\right)(x)$ and $\left(P^\pi\Phi w\right)(x)$ denote the $x$-th coordinate of the corresponding vector $\Phi w\in\mathbb{R}^{|\mathcal{X}|}$ and $P^\pi \Phi w\in\mathbb{R}^{|\mathcal{X}|}$. The expected semi-gradient for $\Phi$ and $w$ can be computed as 
\begin{align}
   \partial_w \mathbb{E}[L] = -\Phi^T D^\pi \left(R^\pi - \left(I-\gamma P^\pi\right)\Phi w\right), \partial_\Phi \mathbb{E}[L] = -D^\pi \left(R^\pi - \left(I-\gamma P^\pi\right)\Phi w\right) w^T.\label{eq:deeptd-expected}
\end{align}
Recall that end-to-end linear TD is defined as follows
\begin{align*}
    \frac{dw_t}{dt}=-\eta_w \cdot \partial_w \mathbb{E}[L],\ \frac{d\Phi_t}{dt} = -\eta_\Phi \cdot \partial_w \mathbb{E}[L].
\end{align*}
Plugging into the formula in \cref{eq:deeptd-expected}, we arrive at the ODE system in \cref{lemmaode}.

Now, to derive the two time-scale ODE dynamics, we just need to replace the generic weight vector $w_t$ in the update to $\Phi_t$ above by the fixed point $w_{\Phi_t}^\ast$. This gives rise to the ODE in \cref{eq:deep-linear-td-ode-two-time-scale}.

\section{Proof}

\lemmaweightederror*
\begin{proof}
Let $v_{t,i}$ denote the $i$-th column of $\Phi w- V^\pi$ for $1\leq i\leq h$. We have
\begin{align*}
    E(\Phi,w) &= \text{Tr}\left(\left(\Phi w - V^\pi\right)^T D^\pi(I-\gamma P^\pi) \left(\Phi w - V^\pi\right)\right) \\
    &= \sum_{i=1}^h v_{t,i}^T D^\pi(I-\gamma P^\pi)v_{t,i}.
\end{align*}
Since $D^\pi(I-\gamma P^\pi)$ is PD \citep{sutton2016emphatic}, we have each of the $i$-th term above is non-negative and is only zero when $v_{t,i}=0$. This concludes the result.
\end{proof}

\theoremvalue*
\begin{proof}
For notational simplicity, let $A=D^\pi(I-\gamma P^\pi)$. Under the reversibility assumption, $A$ is symmetric $A^T=A$. We also define $v_t\coloneqq \Phi_t w_t - V^\pi$ and denote $v_{t,i}$ as its $i$-th column. Define the matrix $M=\left(\Phi w-V^\pi\right)^T A \left(\Phi w-V^\pi\right)$, then note that its $(i,j)$-thn component is
$
    v_{t,i} ^T A v_{t,j}
$. As a result, the weighted error rewrites as
\begin{align*}
    E(\Phi_t,w_t) = \sum_{i=1}^h M_{ii} = \sum_{i=1}^h \underbrace{\frac{1}{2}v_{t,i}^T Av_{t,i}}_{\eqqcolon M_{ii}}.
\end{align*}
Let $w_{t,i}$ be the $i$-th column of matrix $w_t$.
Recall by definition of the end-to-end linear TD dynamics, we have
\begin{align*}
    \frac{d\Phi_t}{dt} &= -\eta_\Phi\cdot A\left(\Phi_t w_t - V^\pi\right) w_t^T = -\eta_w\cdot A \sum_{i=1}^h v_{t,i} w_{t,i}^T \\
    \frac{dw_{t,i}}{dt} &=  -\Phi_t^T A v_{t,i}.
\end{align*}
In order to derive $\partial_{\Phi_t}E(\Phi_t,w_t)$ and $\partial_{w_t}E(\Phi_t,w_t)$, note a few useful facts as follows,
\begin{align*}
    \partial_{\Phi_t} M_{ii} &= \frac{1}{2}(A+A^T)v_{t,i}w_{t,i}^T = Av_{t,i}w_{t,i}^T,\\
    \partial_{w_t} M_{ii} &= \frac{1}{2}\Phi^T (A+A^T)v_{t,i} = \Phi^TAv_{t,i}.
\end{align*}
As a result, we can verify
\begin{align*}
     \frac{d\Phi_t}{dt} &= -\eta_\Phi\sum_{i=1}^h \partial_{\Phi_t} M_{ii} = -\eta_\Phi \partial_{\Phi_t} E(\Phi_t,w_t) \\
     \frac{dw_t}{dt} &= -\eta_w\sum_{i=1}^h \partial_{w_t} M_{ii} = -\eta_w \partial_{w_t} E(\Phi_t,w_t) ,
\end{align*}
where (a) comes from the fact that $A=A^T$. Now with chain rule, we have
\begin{align*}
    \frac{dE(\Phi_t,w_t)}{dt} &=  \text{Tr}\left(\left(\partial_{\Phi_t}E(\Phi_t,w_t)\right)^T \frac{d\Phi_t}{dt}\right) +  \text{Tr}\left(\left(\partial_{w_t}E(\Phi_t,w_t)\right)^T \frac{dw_t}{dt}\right) \\
    &= -\left(\frac{1}{\eta_\Phi}\left\lVert \frac{d\Phi_t}{dt}\right\rVert_2^2 + \frac{1}{\eta_w} \left\lVert \frac{dw_t}{dt}\right\rVert_2^2\right)\leq 0,
\end{align*}
which is strictly negative if $(\Phi_t,w_t)$ is not at a critical point. The proof is hence concluded.
\end{proof}

\lemmacritical*
\begin{proof}
For convenience, define $A\coloneqq D^\pi(I-\gamma P^\pi)$. 
We set $\frac{dw_t}{dt}=0$ and $\frac{d\Phi_t}{dt}=0$. It is straightforward to see that $\frac{dw_t}{dt}=0$ effectively allows us to derive the TD fixed point 
\begin{align*}
    w_t = \left(\Phi_t^T A\Phi_t\right)^{-1}\Phi_t^TAV^\pi \eqqcolon w_{\Phi_t}^\ast. 
\end{align*}
Note that for the above expression, we have used the assumption that $\Phi_t^TA\Phi_t$ is invertible.
Then, solving for $\frac{d\Phi_t}{dt}=0$ and noting that $w_t=w_{\Phi_t}^\ast$, we have
\begin{align*}
    A\Phi_t\left(\Phi_t^TA\Phi_t\right)^{-1}\Phi_t^T AV^\pi(V^\pi)^T A^T\Phi (\Phi_t^T A^T\Phi_t)^{-1} = AV^\pi(V^\pi)^T A^T\Phi (\Phi_t^T A^T\Phi_t)^{-1}.
\end{align*}
Cancelling out the invertible matrix at the end, we have 
\begin{align}
    \underbrace{A\Phi_t\left(\Phi_t^TA\Phi_t\right)^{-1}\Phi_t^T}_{M} AV^\pi(V^\pi)^T A^T\Phi  = AV^\pi(V^\pi)^T A^T\Phi.\label{eq:eq1}
\end{align}
Now, consider the projection matrix $M$, which satisfies $M^2=M$. We exploit the following useful property of the projection matrix
\begin{align*}
    Mx = \arg\min_{y\in\text{span}(A\Phi)}\left\lVert x-y\right\rVert_{(A^T)^{-1}}^2,
\end{align*}
where $\left\lVert x-y\right\rVert_{(A^{-1})^T}^2 \coloneqq (x-y)(A^T)^{-1}(x-y)$ is the squared weighted norm of $x-y$ under the PD matrix $(A^T)^{-1}$ (we show that this matrix is indeed PD towards the end of the proof). Applying the result of  \cref{lemmaprojection} to each column of the vector $AV^\pi(V^\pi)^TA^T \Phi$, we have \cref{eq:eq1} is equivalent to below
\begin{align}
    AV^\pi(V^\pi)^TA^T \Phi \in \text{span}\left(A\Phi\right).\label{eq:eq2}
\end{align}
From \cref{eq:eq2}, we derive a set of equivalent conditions.
\begin{align*}
     AV^\pi(V^\pi)^TA^T \Phi \in \text{span}\left(A\Phi\right)
    &\Longleftrightarrow_{(a)} \exists B\in\mathbb{R}^{k\times k}, \ AV^\pi(V^\pi)^TA^T \Phi = A\Phi B \\
    &\Longleftrightarrow_{(b)} (I-\gamma P^\pi)^{-1}R^\pi(R^\pi)^T (D^\pi)^T \Phi = \Phi B  \\
    &\Longleftrightarrow_{(c)} D^\pi R^\pi(R^\pi)^T (D^\pi)^T \Phi =  A \Phi B  \\
    &\Longleftrightarrow_{(d)} D^\pi R^\pi(R^\pi)^T (D^\pi)^T \Phi \in \text{span}(A\Phi).  
\end{align*}
Here, (a) follows from the definition of $\text{span}(A\Phi)$; (b) follows by canceling $A$ on both sides of the equation; (c) follows from some straightforward algebraic manipulations; (d) follows from the definition of $\text{span}(A\Phi)$. This concludes the proof.

\paragraph{PD of matrix $(A^T)^{-1}$.} Recall that $A$ is PD, and by definition this means $x^TAx\geq 0$ and is only zero for $x=0$. Given any $x$, we want to show 
\begin{align*}
    x^T (A^T)^{-1} x \geq 0
\end{align*}
and is only zero for $x=0$. To see this, for any fixed vector $x$ we define $y=A^{-1}x$. Since $A$ is PD, $A^{-1}$ is invertible and $x=0$ if and only if $y=0$. Now, we have
\begin{align*}
    x^T (A^T)^{-1} x = y^T A y , 
\end{align*}
which is non-negative for all $y\neq 0$ and is only zero when $y=0$, in which case $x=0$. This shows the matrix $(A^T)^{-1}$, though asymmetric in general, is also PD.

\end{proof}

\begin{restatable}{lemma}{lemmaprojection}\label{lemmaprojection}
For any PD matrix $D$ and matrix $A$ with compatible shapes, define the projection matrix $M\coloneqq A(A^TDA)^{-1}A^TD$. Let $v$ be a vector. The following two results are equivalent: (1) $Mv=v$; (2) $v\in\text{span}(A)$.
\end{restatable}
\begin{proof}
By definition of the projection matrix $M$, we can write
\begin{align*}
    Mv = \arg\min_{y\in\text{span}(A)}\left\lVert v-y \right\rVert_D^2.
\end{align*}
Starting with condition (1), we know that $v$ is a feasible solution to the optimization problem above. This means $v\in\text{span}(A)$, which is condition (2). Starting with condition (2), we know that $v$ is a feasible solution to the optimization problem that defines the projection. It is straightforward to verify that $v$ is also the unique optimal solution (unique because $D$ is PD). Hence, we have $Mv=v$.
\end{proof}

\corovalue*
\begin{proof}
Using chain rule as in the proof of \cref{theoremvalue}, for any generic argument $(\Phi_t,w_t)$ to the weighted loss function, we have
\begin{align*}
    \frac{dE(\Phi_t,w_t)}{dt} &=  \text{Tr}\left(\left(\partial_{\Phi_t}E(\Phi_t,w_t)\right)^T \frac{d\Phi_t}{dt}\right) +  \text{Tr}\left(\left(\partial_{w_t}E(\Phi_t,w_t)\right)^T \frac{dw_t}{dt}\right).
\end{align*}
Recall that by definition the TD fixed point $w_{\Phi_t}^\ast$ satisfies the equation
\begin{align*}
    \Phi_t^T D^\pi(I-\gamma P^\pi) (\Phi_tw_{\Phi_t}^\ast-V^\pi) =0.
\end{align*}
Meanwhile, the proof in \cref{theoremvalue} has showed
\begin{align*}
    \partial_{w_t} E(\Phi_t,w_t) = -\Phi_t^T D^\pi(I-\gamma P^\pi) (\Phi_tw_t - V^\pi).
\end{align*}
Combining the above two results, we have $\partial_{w_{\Phi_t}^\ast} E(\Phi_t,w_{\Phi_t}^\ast)=0$. Further, using the result from \cref{theoremvalue} we have $\frac{d\Phi_t}{dt}=-\eta_\Phi\cdot \partial_{\Phi_t}E(\Phi_t,w_{\Phi}^\ast)$. This implies the following
\begin{align*}
    \frac{dE(\Phi_t,w_{\Phi_t}^\ast)}{dt} = -\frac{1}{\eta_\Phi}\left\lVert \frac{d\Phi_t}{dt}\right\rVert_2^2,
\end{align*}
which concludes the proof.
\end{proof}

\lemmaconstant*
\begin{proof}
Our result follows closely the proof technique in \citep{tang2022understanding}. Under the two time-scale TD dynamics, the weight $w_t=w_{\Phi_t}^\ast$ is the fixed point solution given the representation $\Phi_t$. 
Let $A_t=\Phi_tw_t\in\mathbb{R}^{|\mathcal{X}| \times h}$ be the matrix product. The chain rule combined with the first-order optimality condition on $w_t$ (which defines the fixed point) implies
\begin{align}
    \partial_{w_t} \mathbb{E}\left[L(\Phi_t,w_t)\right] = \Phi_t^T \partial_{A_t} \mathbb{E}\left[L(\Phi_t,w_t)\right] = 0.\label{eq:optimality-cond}
\end{align}
On the other hand, the semi-gradient update for $\Phi_t$ can be written as 
\begin{align*}
    \dot{\Phi}_t = -\partial_{A_t}\mathbb{E}\left[L(\Phi_t,w_t)\right](w_t)^T.
\end{align*}
Thanks to \cref{eq:optimality-cond}, we have 
\begin{align*}
    \Phi_t^T \dot{\Phi}_t = -\Phi_t^T \partial_{A_t}\mathbb{E}\left[L(\Phi_t,P_t)\right](P_t)^T = 0.
\end{align*}
Then, taking time derivative on the covariance matrix
\begin{align*}
    \frac{d}{dt} \left(\Phi_t^T\Phi_t\right) = \dot{\Phi}_t^T\Phi_t + \Phi_t^T\dot{\Phi}_t = \left(\Phi_t^T\dot{\Phi}_t\right)^T + \Phi_t^T\dot{\Phi}_t = 0,
\end{align*}
which implies that the covariance matrix is constant along the ODE dynamics.
\end{proof}

\theoremspectral*
\begin{proof}
Under the assumption $P^\pi$ is symmetric, we can verify the uniform distribution $D^\pi=|\mathcal{X}|^{-1}I$ is the stationary distribution. As before, for notational simplicity we let $A=D^\pi(I-\gamma P^\pi)$. Under the reversibility assumption, we have $A^T=A$. Rewriting the TD fixed point with $A$, we have
\begin{align*}
    w_\Phi^\ast = \left(\Phi^TA\Phi\right)^{-1}\Phi^TAV^\pi.
\end{align*}
Plugging $w_\Phi^\ast$ into \cref{eq:deep-linear-td-ode-two-time-scale}, the aggregate dynamics to $\Phi_t$ is
\begin{align*}
    \frac{d\Phi_t}{dt} = \eta_\Phi \cdot \left(I-A\Phi_t\left(\Phi_t^TA\Phi_t\right)^{-1}\Phi_t^T\right) AV^\pi(V^\pi)^T A^T\Phi_t \left(\Phi_t^TA^T\Phi_t\right)^{-1}.
\end{align*}
When $R^\pi(R^\pi)^{-1}=I$, we have
\begin{align*}
    V^\pi(V^\pi)^T = (I-\gamma P^\pi)^{-1}\left((I-\gamma P^\pi)^{-1}\right)^T = |\mathcal{X}|^{-2} A^{-1}(A^{-1})^T.
\end{align*}
Plugging this back into the dynamics for $\Phi_t$, we have
\begin{align*}
    \frac{d\Phi_t}{dt} = \eta_\Phi|\mathcal{X}|^{-2}\cdot \left(I-A\Phi_t\left(\Phi_t^TA\Phi_t\right)^{-1}\Phi_t^T\right) A A^{-1} (A^{-1})^T A\Phi_t \left(\Phi_t^TA^T\Phi_t\right)^{-1}.
\end{align*}
Now, consider the trace objective $f(\Phi_t)=\text{Tr}\left(\Phi_t^T (I-\gamma P^\pi)^{-1}\Phi_t\right) = |\mathcal{X}|^{-1}\text{Tr}\left(\Phi_t^T A^{-1} \Phi_t\right)$. Note that since $(\Phi_t^T A^T \Phi_t)^{-1}$ is symmetric and PD, we can write  $(\Phi_t^T A^T \Phi_t)^{-1}=LL^T$. Then we consider its time derivative
\begin{align*}
    \frac{df}{dt} &= \text{Tr}\left(\Phi_t^T A^{-1} \frac{d\Phi_t}{dt}\right) \\
    &= 2|\mathcal{X}|^{-1}\cdot\text{Tr}\left(\Phi_t^T A^{-1} \frac{d\Phi_t}{dt}\right) \\
    &= 2\eta_\Phi |\mathcal{X}|^{-3} \cdot \text{Tr}\left(\Phi_t^T A^{-1} \left(I-A\Phi_t\left(\Phi_t^TA\Phi_t\right)^{-1}\Phi_t^T\right) A A^{-1} (A^{-1})^T A\Phi_t \left(\Phi_t^TA^T\Phi_t\right)^{-1}\right) \\
    &= 2\eta_\Phi |\mathcal{X}|^{-3} \cdot \text{Tr}\left(\Phi_t^T (\sqrt{A})^{-1} \left(I-\sqrt{A}\Phi_t\left(\Phi_t^TA\Phi_t\right)^{-1}\Phi_t^T\sqrt{A}\right)  (\sqrt{A})^{-1}\Phi_t LL^T\right) \\
    &= 2\eta_\Phi |\mathcal{X}|^{-3} \cdot \text{Tr}\left(\Phi_t^T (\sqrt{A})^{-1} \left(I-\sqrt{A}\Phi_t\left(\Phi_t^TA\Phi_t\right)^{-1}\Phi_t^T\sqrt{A}\right)  \underbrace{(\sqrt{A})^{-1}\Phi_t L}_{\eqqcolon v}L^T\right) \\
    &= 2\eta_\Phi |\mathcal{X}|^{-3} \cdot \text{Tr}\left(v^T \left(I-\underbrace{\sqrt{A}\Phi_t\left(\Phi_t^TA\Phi_t\right)^{-1}\Phi_t^T\sqrt{A}}_{\eqqcolon M}\right) v \right) \\
\end{align*}
Note the matrix $M$ is a projection matrix which satisfies $M^T=M,M^T=M$. This means $\text{Tr}\left(v^T(I-M)v\right)\geq 0$ for any matrix $v$. With the above definitions, we can rewrite the update
\begin{align*}
    \frac{d\Phi_t}{dt} = \sqrt{A} (I-M)vL^T.
\end{align*}
Now, if $\frac{d\Phi_t}{dt} \neq 0$, this implies $(I-M)v\neq 0$ since both $LL^T$ and $A$ are PD. This further implies $\frac{df}{dt}$ is strictly positive. Meanwhile, if $ \frac{d\Phi_t}{dt}=0$, this means $(I-M)v=0$, which further implies $\frac{df}{dt}=0$. 

\end{proof}

\corocritical*
\begin{proof}
Define $A=D^\pi(I-\gamma P^\pi)$. 
When $\Phi_0^T A \Phi_0$ is assumed invertible, and since $A$ is PD, $\Phi$ must be of rank $k$. By \cref{lemmaconstant}, the rank of $\Phi$ is preserved over time and hence $\Phi_t^T A\Phi_t$ is always invertible along the ODE dynamics. This means the conditions for \cref{lemmacritical} is are all satisfied and hence,
\begin{align*}
    D^\pi R^\pi (D^\pi R^\pi)^T \Phi \in \text{span}\left(A\Phi\right).
\end{align*}
When $P^\pi$ is symmetric and $R^\pi(R^\pi)^T=I$, we deduce $D^\pi=|\mathcal{X}|^{-1}I$ and the following specialized condition holds,
\begin{align*}
    \Phi \in \text{span}\left(A\Phi \right).
\end{align*}
Since $\Phi$ is rank $k$, the above condition implies there exists an invertible matrix $B\in\mathbb{R}^{k\times k}$,
\begin{align*}
    \Phi = A\Phi B.
\end{align*}
This in turn implies $A\Phi = \Phi B^{-1}$ and equivalently, $A\Phi \in \text{span}(\Phi)$. When $P^\pi$ is symmetric, $D^\pi=|\mathcal{X}|^{-1}I $ and we have from before,
$
    \Phi - \gamma P^\pi \Phi \in \text{span}(\Phi)$.
This implies $P^\pi \Phi = \gamma^{-1} \Phi (I-B)$ and hence $P^\pi \Phi \in \text{span}(\Phi)$.

Now, assume $P^\pi\Phi \in \text{span}(\Phi)$, we seek to show that $\Phi$ is a critical point of the dynamics. The assumption directly implies that there exists matrix $B\in\mathbb{R}^{k\times k}$ such that
$
    P^\pi \Phi = \Phi B.
$.
This implies
\begin{align*}
    (I-\gamma P^\pi) \Phi = \Phi (I-\gamma B).
\end{align*}
Now, since $\Phi^T\Phi=\Phi_0^T\Phi_0$ and because $\Phi_0$ is initialized of rank $k$, $\Phi$ is of rank $k$ too. Because the matrix $(I-\gamma P^\pi)$ is full rank, it must be that $(I-\gamma B)$ is of rank $k$ as well, and hence invertible. This gives
\begin{align*}
    \Phi = (I-\gamma P^\pi) \Phi (I-\gamma B)^{-1} \in \text{span}(A\Phi).
\end{align*}
\end{proof}

\lemmaexample*
\begin{proof}
Recall $R_n^\pi$ to be the $n$-th column of the reward function matrix $R^\pi$. We have
\begin{align*}
    R^\pi(R^\pi)^T = \sum_{n=1}^h R_n^\pi(R_n^\pi)^T 
\end{align*}
Consider the $(i,j)$-th component of the matrix
\begin{align*}
    \left[R^\pi(R^\pi)^T\right](i,j) = \sum_{n=1}^h \left[R_n^\pi(R_n^\pi)^T \right](i,j) = \sum_{n=1}^h R_{ni}^\pi R_{nj}^\pi,
\end{align*}
where $R_{ni}^\pi$ denotes the $i$-th component of the column vector $R_n^\pi$ for $1\leq i\leq |\mathcal{X}|$. When $i=j$, the summation is over a set of $h$ i.i.d. random variables each with mean $\sigma^2/h$. When $i\neq j$, the summation is over
a set of $h$ i.i.d. random variables each with mean zero. The law of large number concludes the proof.
\end{proof}

\section{Discussion on related work} \label{appendix:related}
Here, we provide a more comprehensive discussion on prior work.

\paragraph{From linear TD to end-to-end linear TD.} Since the introduction of TD-learning \citep{sutton1988learning}, there have been numerous efforts at understanding the algorithm. The seminal work of \citep{tsitsiklis1996analysis} proposes to understand linear TD through its expected continuous time behavior, characterized by a linear ODE system. A number of follow-up work has applied similar techniques to understanding the stability of, e.g.,  Q-learning \citep{melo2008analysis}, off-policy TD \citep{sutton2016emphatic}, among others problems.

Much of the prior work has focused on the classic linear TD setup, where the representations are assumed fixed throughout learning.
Closely related to our work is \citep{lyle2021effect} where they proposed to understand the learning dynamics of end-to-end linear TD through its corresponding ODE system. Since such an ODE system is highly non-linear, it is more challenging to provide generic characterizations without restrictive assumptions. \citet{lyle2021effect} bypasses the non-linearity issue by essentially assuming a fixed weight parameter $w_t\equiv w$. In light of \cref{eq:deep-linear-td-ode}, they study the following dynamics
\begin{align*}
    \frac{d\Phi_t}{dt} = \eta_\Phi \cdot D^\pi\left(R^\pi -(I-\gamma P^\pi)\Phi_t w\right)w^T
\end{align*}
which effectively reduces  to a linear system in $\Phi_t$ and is more amenable to analysis.

Our work differs in a few important aspects. Firstly, our analysis adheres strictly to the vanilla end-to-end linear TD  (\cref{eq:deep-linear-td-ode}) or two-time scale end-to-end linear TD  (\cref{eq:deep-linear-td-ode-two-time-scale}), without imposing the constant weight assumption as in \citep{lyle2019comparative}. Our analysis is also slightly more general, as it consists in constructing a few scalar functions that characterize the dynamics. This is more applicable to general ODEs where obtaining exact solutions is not tractable.

\paragraph{TD with non-linear function approximations.} Going beyond the linear parameterization, a number of prior work considered the problem of non-linear TD where the value function is represented as some generic smooth non-linear functions. For example, \citet{maei2009convergent} established the first convergence result of smooth function classes with two time-scale updates. A number of concurrent work \citet{agazzi2022temporal,cai2019neural,sirignano2022asymptotics}, each with subtly different theoretical setups in place, showed convergence of TD-learning in the over-parameterized regimes, e.g., when the value functions are approximated by very wide neural networks. In this case, representations are confined to be near the initialized values and cannot evolve much over time. To address the limitation, \citet{zhang2020can} generalized the result from \citep{cai2019neural} and showed that in the over-parameterized regimes, the representations converge to the near optimal ones with the lowest projected value prediction error. Note that \cref{theoremvalue} echoes this result and provides a complementary result under the end-to-end linear TD setup.

Another closely related work is \citep{brandfonbrener2019geometric}, where they established the convergence behavior of TD-learning with smooth homogeneous functions. A key requirement underlying their result is that the MDP is sufficiently reversible, which echos the assumption we make in \cref{theoremvalue}. However, under their framework there is no clear notion of representation as defined in the bi-linear case. An interesting future direction would be to study representation dynamics with non-linear function approximations.

\paragraph{Representations learning via self-predictive learning.} Motivated by value-based TD-learning, self-predictive learning directly employs the notion of bootstrapping to the representation space, which has produced a number of empirically successful implementations \citep{guo2020bootstrap,schwarzer2021dataefficient,guo2020bootstrap}. The high level idea is to minimize the prediction error
\begin{align*}
    \left\lVert P\left(\phi_{x_t}\right) - \text{sg}\left(\phi_{x_{t+1}}\right) \right\rVert_2^2,
\end{align*}
where $P:\mathbb{R}^k\rightarrow\mathbb{R}^k$ is the learned transition dynamics in the representation space. Recently, Tang et al. \citep{tang2022understanding} proposes to understand the behavior of self-predictive learning through its corresponding ODE. From their discussion, we note that two time-scale dynamics is a generic way to enforcing optimization constraints. Our results in \cref{sec:spectral} build on this and show a similar characterization for the end-to-end linear TD dynamics.

\paragraph{Two time-scale learning dynamics.} The idea of two time-scale learning dynamics is not new in TD-learning. \citet{levine2017shallow} proposed a combination of LSTD updates for the linear weights on top of DQN representations, which are shaped by gradient descents. In continuous time, this is effectively the two time-scale dynamics (Eqn~\eqref{eq:deep-linear-td-ode-two-time-scale}). While their discussion is empirically motivated, we focus on the representation learning aspect of such a learning dynamics.

\paragraph{Auxiliary tasks.} In deep RL literature, it has been empirically observed that certain extra training objectives, normally referred to as \emph{auxiliary tasks}, are useful for improving the agent performance on the objective of interest such as optimizing the cumulative returns \citep{jaderberg2016reinforcement,bellemare2019geometric,dabney2021value}. Though the conventional wisdom is that such auxiliary tasks are useful in shaping representations, it is generally difficult to characterize exactly what representations are induced, despite some recent theoretical efforts \citep{lyle2021effect}. Our results demonstrate the potential benefits of random reward predictions when combined with TD-learning \citep{dabney2021value}, by connecting the learning dynamics to gradient-based spectral decomposition of the transition matrix.

\section{Experiment details}
\label{appendix:exp}

We provide further experimen details on the tabular and deep RL experiments in the main paper.

\subsection{Tabular experiments}
All tabular experiments are conducted on random MDPs with $|\mathcal{X}|=30$ states. The general transition matrix $P^\pi$ is generated as follows,
\begin{align*}
    P^\pi = \alpha P_\text{perm} + (1-\alpha) P_\text{ds},
\end{align*}
where $P_\text{ds}$ is a randomly sampled doubly-stochastic matrix and $P_\text{perm}$ is a randomly generated permutation matrix. We set $\alpha=0.95$ so that $P^\pi$ is likely to violate the reversibility assumption. In the symmetric case, the transition matrix is computed as
\begin{align*}
    P^\pi = (P_\text{ds} + P_\text{ds}^T) / 2.
\end{align*}
The doubly-stochastic matrix is randomly generated based on the procedure in \citep{tang2022understanding}. Each entry of the reward function $R^\pi$ is randomly sampled from $\mathcal{N}(0,1)$

\paragraph{Normalized trace objective.}
For any matrix $P^\pi$, the trace objective is computed as $f(\Phi_t)=\text{tr}\left(\Phi_t^T A\Phi_t\right)$ for $A=(I-\gamma P^\pi)^{-1}$. To calculate the normalized objective, we compute the baseline value $\bar{f}$ as the sum of the top $k$ eigenvalues of the symmetrized matrix $(A+A^T)/2$. The normalized trace objective is $\bar{f}(\Phi_t)\coloneqq f(\Phi_t) / \bar{f}$. When $P^\pi$ is symmetric, $\bar{f}(\Phi_t)$ is upper bounded by $1$.

\paragraph{Unrolling the ODE dynamics.}
All results are based on the solving the exact ODE dynamics, using the Scipy ODE solver \citep{2020SciPy-NMeth}. Throughout, we initialzie the representation $\Phi_t$ matrix as orthonormal. We start by generating $k$ column vectors of size $\mathbb{R}^{|\mathcal{X}|}$, with each entry randomly generated from $\mathcal{N}(0,1)$. Then we apply the Gram-Schmidt orthogonalization procedure to the columns to compute the initialized representation.

\subsection{Deep RL experiments}

We use DQN \citep{mnih2013} as a baseline and evaluate all algorithmic variants over $15$ games in the Atari game suite \citep{bellemare2013arcade}. Our testbed is a subset of $15$ Atari games \citep{bellemare2013arcade} on which it has been shown that DQN can achieve reasonable performance, see e.g., \citep{schaul2015prioritized} for how they select the subset of the games: \texttt{asterix}, \texttt{boxing}, \texttt{breakout}, \texttt{freeway}, \texttt{gopher}, \texttt{gravitar}, \texttt{hero}, \texttt{ms pacman}, \texttt{pong}, \texttt{qbert}, \texttt{riverraid}, \texttt{seaquest}, \texttt{skiing}, \texttt{space invaders} and \texttt{venture}. 

\paragraph{Random reward functions.} To generate random reward functions $R_i^\pi(x,a),1\leq i\leq h$, we initialize $h$ networks randomly and directly use the outputs This is similar to the practice of \citep{dabney2021value} except that we do not apply additional activations (such as $\text{sigmoid}$ or \text{tanh} as done in \citep{dabney2021value}) on top. Throughout, we use the Adam optimizer \citep{kingma2014adam} with a fixed learning rate $\eta$, see \citep{mnih2013} for details of other hyper-parameters. 

\paragraph{Hyper-parameters.}
We tune the learning rate $\eta\in\{0.00025,0.0001,0.00005\}$ as suggested in \citep{dabney2021value}. The default DQN uses $\eta=0.00025$. We find that at $\eta=0.0001$ the tuned DQN performs the best. For the auxiliary task, we tune the number of random rewards $h\in\{4,16,64,256\}$. We find that $h=16$ performs slightly better than other alternatives.

\begin{figure}[t]
    \centering
    \includegraphics[keepaspectratio,width=.45\textwidth]{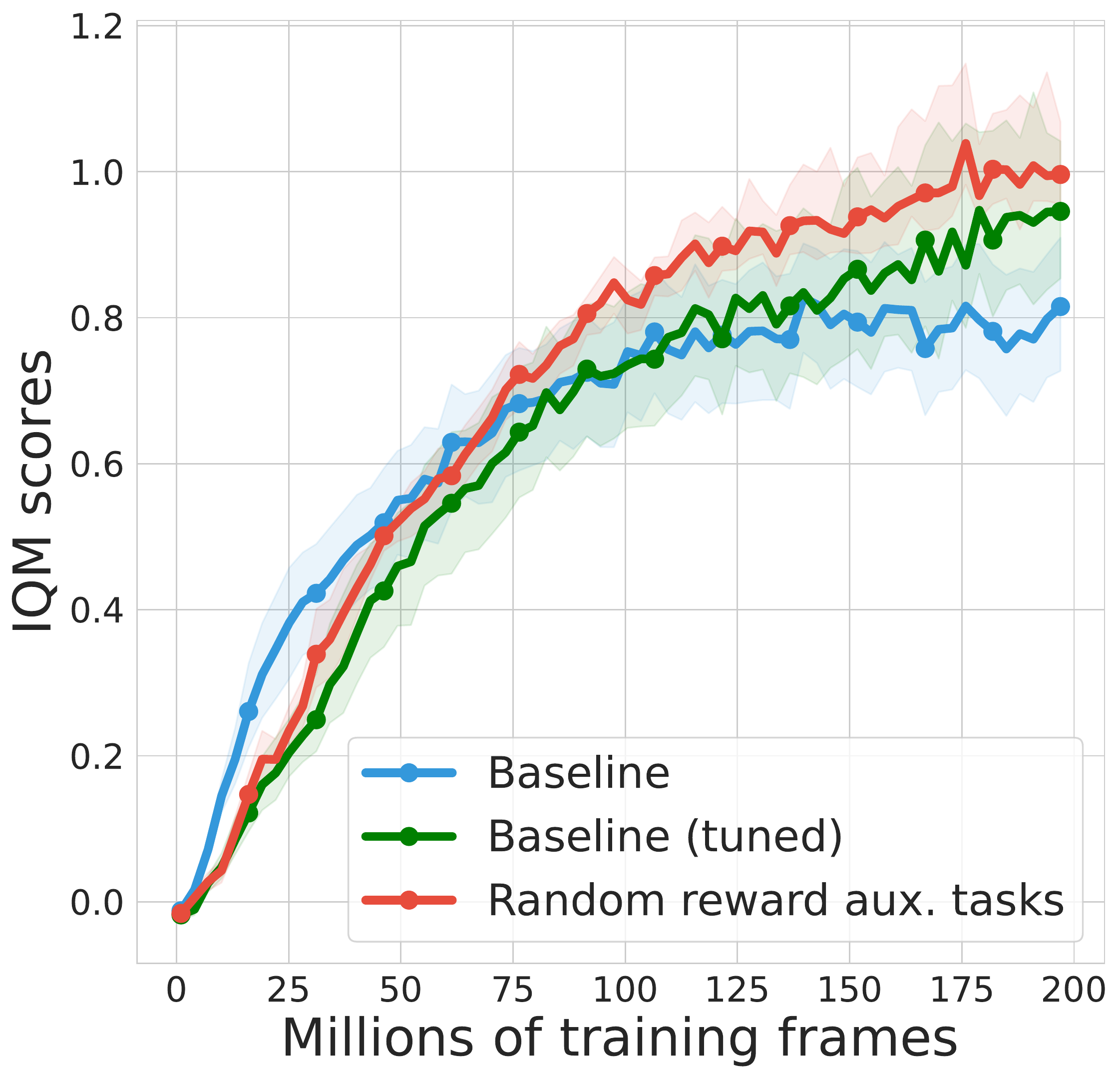}
     \caption{Comparison of training curves of IQM human normalized scores across 15 Atari games. We compare the baseline DQN, DQN with tuned learning rate and DQN with random value predictions as auxiliary tasks. The shaded areas show the $95\%$ bootstrapped confidence intervals averaged over $3$ seeds. The random value prediction tasks seem to provide marginal benefits over the tuned DQN. The IQM score is computed by truncating the top and bottom $25\%$ of scores, averaged across all games and all seeds \citep{agarwal2021deep}.}
    \label{fig:atari-median}
\end{figure}

\end{appendix}

\end{document}